\documentclass[twoside]{article}
\pdfoutput=1
\usepackage[accepted]{aistats2021}
\usepackage{amsmath,amsfonts,bm}

\def\eqref#1{equation~\ref{#1}}

\def\1{\bm{1}}

\def\eps{{\epsilon}}

\DeclareMathAlphabet{\mathsfit}{\encodingdefault}{\sfdefault}{m}{sl}
\SetMathAlphabet{\mathsfit}{bold}{\encodingdefault}{\sfdefault}{bx}{n}

\newcommand{\E}{\mathbb{E}}

\newcommand{\R}{\mathbb{R}}

   \usepackage{xcolor}
\usepackage[colorlinks=true]{hyperref}
\definecolor{dmorange500}{HTML}{FF5F19}
\definecolor{dmblue300}{HTML}{2267EB}
\definecolor{dmred300}{HTML}{FF617B}
\hypersetup{
	colorlinks,
	citecolor=dmblue300,
	linkcolor=dmred300,
	urlcolor=dmorange500}
\usepackage{natbib}

\usepackage[utf8]{inputenc} 
\usepackage[T1]{fontenc}    
\usepackage{url}            
\usepackage{booktabs}       
\usepackage{amsfonts}       
\usepackage{nicefrac}       
\usepackage{microtype}      
\usepackage{multirow}
\usepackage{xspace}
\usepackage{graphicx}
\usepackage{wrapfig}
\usepackage{subcaption}
\usepackage[ruled,vlined,linesnumbered]{algorithm2e}
\usepackage{amsmath,amssymb,amsthm}
\usepackage{thm-restate}

\usepackage[capitalise]{cleveref}
\Crefformat{section}{Section~#2#1#3}
\Crefformat{proposition}{Proposition~#2#1#3}
\Crefformat{corollary}{Corollary~#2#1#3}
\Crefformat{lemma}{Lemma~#2#1#3}
\Crefformat{equation}{Eq.\,#2#1#3} 
\Crefformat{figure}{Fig.\,#2#1#3} 

\SetKwInput{KwInputs}{Inputs}

\usepackage{selectp}

\runningauthor{Guo, Azar, Saade, Thakoor, Piot, Pires, Valko, Mesnard, Lattimore, Munos}
\begin{document}

\newcommand{\Exp}{\mathds{E}}
\newcommand{\Expk}{\mathds{E}_{k}}
\newcommand{\Nat}{\mathbb{N}}
\newcommand{\Ind}{\mathds{1}}
\newcommand{\Rmax}{R_{\rm max}}
\newcommand{\riskyopt}{\succcurlyeq_{\rm ro}}
\newcommand{\cid}{\succcurlyeq_{\rm CID}}
\newcommand{\so}{\succcurlyeq_{\rm so}}
\newcommand{\single}{\succcurlyeq_{\rm sc}}
\newcommand{\ssd}{\succcurlyeq_{\rm ssd}}
\newcommand{\pp}{\mathrel{+}\mathrel{+}}
\newcommand{\Xc}{\mathcal{X}}
\newcommand{\Yc}{\mathcal{Y}}
\newcommand{\Pc}{\mathcal{P}}
\newcommand{\Qc}{\mathcal{Q}}
\newcommand{\Ec}{\mathcal{E}}
\newcommand{\Fc}{\mathcal{F}}
\newcommand{\Gc}{\mathcal{G}}
\newcommand{\Rc}{\mathcal{R}}
\newcommand{\Sc}{\mathcal{S}}
\newcommand{\Ac}{\mathcal{A}}
\newcommand{\Mc}{\mathcal{M}}
\newcommand{\Tc}{\mathcal{T}}
\newcommand{\Vc}{\mathcal{V}}
\newcommand{\Dc}{\mathcal{D}}
\newcommand{\Bc}{\mathcal{B}}
\newcommand{\Hc}{\mathcal{H}}
\newcommand{\A}{\mathcal A}
\renewcommand{\S}{\mathcal S}
\newcommand{\X}{\mathcal X}
\newcommand{\D}{\mathcal D}
\newcommand{\G}{\mathcal G}
\newcommand{\K}{\mathcal K}
\newcommand{\calP}{\mathcal P}
\newcommand{\calI}{\mathcal I}
\newcommand{\barH}{\overline{H}}
\newcommand{\hh}{\hat h}
\renewcommand{\L}{\mathcal L}
\newcommand{\Hyp}{\mathcal H}
\newcommand{\Y}{\mathcal Y}
\newcommand{\B}{\mathcal B}
\newcommand{\C}{\mathcal C}
\newcommand{\F}{\mathcal F}
\newcommand{\W}{\mathcal W}
\newcommand{\Z}{\mathcal Z}
\newcommand{\calE}{\mathcal E}
\newcommand{\calS}{\mathcal{S}}
\newcommand{\calO}{\mathcal{O}}
\newcommand{\V}{\mathbb V}
\newcommand{\Prob}{\mathbb P}
\newcommand{\I}{\mathbb I}
\newcommand{\N}{\mathcal N}
\newcommand{\balpha}{\boldsymbol \alpha}
\newcommand{\bmu}{\boldsymbol \mu}
\newcommand{\bSigma}{\boldsymbol \Sigma}
\newcommand{\bP}{\mathbf{P}}
\newcommand{\bhP}{\widehat{\mathbf{P}}}
\newcommand{\bT}{\boldsymbol{T}}
\newcommand{\bX}{\boldsymbol{X}}
\newcommand{\bY}{\boldsymbol{Y}}
\newcommand{\bx}{\boldsymbol{x}}
\newcommand{\MV}{\textnormal{MV}}
\newcommand{\hMV}{\widehat{\textnormal{MV}}}
\newcommand{\barmu}{\bar\mu}
\newcommand{\hpi}{\hat\pi}
\newcommand{\tDelta}{\widetilde{\Delta}}
\newcommand{\hmu}{\widehat{\mu}}
\newcommand{\hrho}{\hat\rho}
\newcommand{\heps}{\hat\eps}
\newcommand{\hnu}{\hat\nu}
\newcommand{\trho}{\tilde\rho}
\newcommand{\brho}{\bar\rho}
\newcommand{\hM}{\widehat{M}}
\newcommand{\hN}{\widehat{N}}
\newcommand{\tmu}{\widetilde{\mu}}
\newcommand{\tpi}{\widetilde{\pi}}
\newcommand{\barvar}{\bar\sigma^2}
\newcommand{\tvar}{\tilde\sigma^2}
\newcommand{\htheta}{\hat{\theta}}
\newcommand{\hR}{\widehat{\mathcal{R}}}
\newcommand{\invdelta}{1/\delta}
\newcommand{\boldR}{\mathbb R}
\newcommand{\mvlcb}{\textnormal{\texttt{MV-LCB }}}
\newcommand{\ucb}{\textnormal{\textsl{UCB }}}
\newcommand{\ucbv}{\textnormal{\textsl{UCB-V }}}
\newcommand{\mvlcbt}{\textnormal{\texttt{MV-LCB(t) }}}
\newcommand{\mom}{\textnormal{MoM}}
\newcommand{\me}{\textnormal{ME}}
\newcommand{\mt}{\textnormal{MT}}
\newcommand{\eh}{1/(1-\gamma)}
\newcommand{\ehf}{\frac1{1-\gamma}}

\newcommand{\cvar}{\textnormal{C}}
\newcommand{\hcvar}{\widehat{\textnormal{C}}}
\newcommand{\hvar}{\widehat{\textnormal{V}}}

\newcommand{\avg}[2]{\frac{1}{#2} \sum_{#1=1}^{#2}}
\newcommand{\hDelta}{\widehat{\Delta}}
\newcommand{\hGamma}{\widehat{\Gamma}}

\newcommand{\beq}{\begin{equation}}
\newcommand{\eeq}{\end{equation}}

\newcommand{\beqa}{\begin{eqnarray}}
\newcommand{\eeqa}{\end{eqnarray}}

\newcommand{\beqan}{\begin{eqnarray*}}
\newcommand{\eeqan}{\end{eqnarray*}}

\renewcommand{\P}{\mathbb{P}}
\renewcommand{\Pr}{\mathbb{P}}
\newcommand{\Q}{\mathbb{Q}}
\newcommand{\Esp}{\mathbb{E}}
\newcommand{\indic}[1]{\mathbb{I}\{#1\}}
\newcommand{\EE}[1]{\E\left[#1\right]}
\newcommand{\wh}{\widehat}
\newcommand{\wt}{\widetilde}

\let\R\undefined 
\newcommand{\R}{\mathbb{R}}
\newcommand{\Real}{\mathbb{R}}
\newcommand{\Normal}{\mathcal{N}}

\newcommand{\eqdef}{\stackrel{\rm def}{=}}

\newcommand{\cl}[2][ (]{
\ifthenelse{\equal{#1}{ (}}{\left (#2\right)}{}
\ifthenelse{\equal{#1}{[}}{\left[#2\right]}{}
\ifthenelse{\equal{#1}{\{}}{\left\{#2\right\}}{}
}

\newcommand{\inset}[3][C]{
\ifthenelse{\equal{#1}{C}}{{#2}\in\mathcal{#3}}{}
\ifthenelse{\equal{#1}{T}}{{#2}\in{#3}}{}
}

\newcommand{\ESum}[4][C]
{\ifthenelse{\equal{#1}{C}}{\underset{\inset{#3}{#4}}{\sum}#2}{}
 \ifthenelse{\equal{#1}{T}}{\underset{\inset[N]{#3}{#4}}{\sum}#2}{}
\ifthenelse{\equal{#1}{U}}{\overset{#4}{\underset{#3}{\sum}}#2}{}
\ifthenelse{\equal{#1}{X}}{\sideset{}{_{#3}^{#4}}{\sum}#2}{}
\ifthenelse{\equal{#1}{S}}{\sideset{}{_{#3}^{#4}}{\sum}#2}{}
\ifthenelse{\equal{#1}{O}}{{\underset{ (#3,#4)}{\sum}}#2}{}
\ifthenelse{\equal{#1}{I}}{{\underset{#3}{\sum}}#2}{}}

\newcommand{\VF}[2][N]{
\ifthenelse{\equal{#1}{L}}{V^{\pi}_{\lambda} (#2)}{}
\ifthenelse{\equal{#1}{C}}{V{^{\pi} (#2)}}{}
\ifthenelse{\equal{#1}{T}}{V^*_{\bar{\pi}} (#2)}{}
\ifthenelse{\equal{#1}{CO}}{V{^{*} (#2)}}{}
\ifthenelse{\equal{#1}{Mx}}{V^{\pi}_{\infty} (#2)}{}
\ifthenelse{\equal{#1}{Mxo}}{V^{\pi^*}_{\infty} (#2)}{}
\ifthenelse{\equal{#1}{Mn}}{V^{\pi}_{-\infty} (#2)}{}
\ifthenelse{\equal{#1}{Mno}}{V^{\pi^*}_{-\infty} (#2)}{}
}

\newcommand{\Eval}[1][null]{
\ifthenelse{\equal{#1}{null}}{\mathbb{E}}{\mathbb{E}_{#1}}
}

\newcommand{\M}[1][]{
\mathcal{M}_{#1}
}

\newcommand{\qv}[1][null]{
\ifthenelse{\equal{#1}{null}}{Q^*}{Q^{#1}}
}

\newcommand{\T}[1][null]{
\ifthenelse{\equal{#1}{null}}{\mathcal{T}}{\mathcal{T}^{#1}}
}

\newcommand{\subLim}[2]{
\underset{#1\rightarrow#2}{\lim}
}

\newcommand{\Norm}[2][]
{
\left\|#2\right\|_{#1}
}

\newcommand{\bldsym}{\boldsymbol}

\newtheorem{assumption}{Assumption}

\newcommand{\TODO}[1]{(\textbf{TODO: {#1}})}

\let\originalleft\left
\let\originalright\right
\renewcommand{\left}{\mathopen{}\mathclose\bgroup\originalleft}
\renewcommand{\right}{\aftergroup\egroup\originalright}

\renewcommand{\ttdefault}{lmtt}
\newcommand{\COUNT}{\normalfont {\texttt{COUNT}}\xspace}
\newcommand{\RND}{\normalfont {\texttt{RND}}\xspace}
\newcommand{\AR}{\normalfont \texttt{AR}\xspace}
\newcommand{\ADAM}{\normalfont \texttt{ADAM}\xspace}
\newcommand{\oldTACOMax}{\normalfont \texttt{TACOMax}\xspace}
\newcommand{\oldTACOMaxFullName}{\textbf{T}sallis \textbf{A}uto-\textbf{C}ontrastive \textbf{O}ptimisation for \textbf{Max}imum state-visitation entropy exploration\xspace}

\newcommand{\oldGACOMax}{\normalfont \texttt{GACOMax}\xspace}
\newcommand{\oldGACOMaxFullName}{\textbf{G}eometry-\textbf{A}ware-\textbf{C}ontrastive \textbf{O}ptimisation for \textbf{Max}imum state-visitation entropy exploration\xspace}

\newcommand{\GEM}{\normalfont \texttt{GEM}\xspace}
\newcommand{\GEMFullName}{\textbf{G}eometric \textbf{E}ntropy \textbf{M}aximisation\xspace}

\definecolor{graphicbackground}{rgb}{0.96,0.96,0.8}
\definecolor{rouge1}{RGB}{226,0,38}  
\definecolor{orange1}{RGB}{243,154,38}  
\definecolor{jaune}{RGB}{254,205,27}  
\definecolor{blanc}{RGB}{255,255,255} 
\definecolor{rouge2}{RGB}{230,68,57}  
\definecolor{orange2}{RGB}{236,117,40}  
\definecolor{taupe}{RGB}{134,113,127} 
\definecolor{gris}{RGB}{91,94,111} 
\definecolor{bleu1}{RGB}{38,109,131} 
\definecolor{bleu2}{RGB}{28,50,114} 
\definecolor{vert1}{RGB}{133,146,66} 
\definecolor{vert3}{RGB}{20,200,66} 
\definecolor{vert2}{RGB}{157,193,7} 
\definecolor{darkyellow}{RGB}{233,165,0}  
\definecolor{lightgray}{rgb}{0.9,0.9,0.9}
\definecolor{darkgray}{rgb}{0.6,0.6,0.6}
\definecolor{babyblue}{rgb}{0.54, 0.81, 0.94}
\definecolor{citrine}{rgb}{0.89, 0.82, 0.04}
\definecolor{misogreen}{rgb}{0.25,0.6,0.0}
\definecolor{PalePurp}{rgb}{0.66,0.57,0.66}
\definecolor{todocolor}{rgb}{0.66,0.99,0.99}
\definecolor{pearOne}{HTML}{2C3E50}
\definecolor{pearTwo}{HTML}{A9CF54}
\definecolor{pearTwoT}{HTML}{C2895B}
\definecolor{pearThree}{HTML}{E74C3C}
\colorlet{titleTh}{pearOne}
\colorlet{bull}{pearTwo}
\definecolor{pearcomp}{HTML}{B97E29}
\definecolor{pearFour}{HTML}{588F27}
\definecolor{pearFith}{HTML}{ECF0F1}
\definecolor{pearDark}{HTML}{2980B9}
\definecolor{pearDarker}{HTML}{1D2DEC}

\twocolumn[

\aistatstitle{Geometric Entropic Exploration}

\aistatsauthor{ Zhaohan Daniel Guo \And Mohammad Gheshlaghi Azar \And Alaa Saade}
\aistatsaddress{ danielguo@google.com, DeepMind \And DeepMind \And DeepMind }
\aistatsauthor{ Shantanu Thakoor \And Bilal Piot \And Bernardo Avila Pires \And Michal Valko}
\aistatsaddress{ DeepMind \And DeepMind \And DeepMind \And DeepMind}
\aistatsauthor{ Thomas Mesnard  \And  Tor Lattimore \And  R\'emi Munos}
\aistatsaddress{ DeepMind \And DeepMind \And DeepMind }
]

\begin{abstract}
Exploration is essential for solving complex Reinforcement Learning (RL) tasks. Maximum State-Visitation Entropy (MSVE) formulates the exploration problem as a well-defined policy optimization problem whose solution aims at visiting all states as uniformly as possible. This is in contrast to standard uncertainty-based approaches where exploration is transient and eventually vanishes. However, existing approaches to MSVE are theoretically justified only for discrete state-spaces as they are oblivious to the geometry of continuous domains. We address this challenge by introducing \textbf{G}eometric \textbf{E}ntropy  \textbf{M}aximisation (\GEM), a new algorithm that maximises the geometry-aware Shannon entropy of state-visits in both discrete and continuous domains. Our key theoretical contribution is casting geometry-aware MSVE exploration as a tractable problem of optimising a simple and novel noise-contrastive objective function. In our experiments, we show the efficiency of \GEM in solving several RL problems with sparse rewards, compared against other deep RL exploration approaches.
\end{abstract}

\section{Introduction}

Exploration is fundamental for reinforcement learning (RL)~\citep{sutton1998reinforcement} agents to discover new rewarding states and ultimately find an optimal policy. 
In tabular settings, there exist provably efficient exploration methods based on the idea of giving reward bonuses to less explored or novel states (optimism in the face of uncertainty)~\citep{Kearns2002, brafman2002r, kakade2003sample, strehl2006pac, lattimore2014near, dann2015sample, azar2017minimax, cohen2020near, tarbouriech2019no}). In practice, scaling these methods beyond small, tabular settings requires using function approximators for estimating uncertainty~\citep{bellemare2016unifying,osband2016deep, ostrovski2017count, pathak2017curiosity, osband2018randomized,burda2018exploration}.
As these uncertainty estimates are non-stationary and vanishing, this requires careful tweaking of the learning process to remain stable while allowing the policies produced by the RL algorithm to adequately explore the environment before the novelty incentive vanishes. 

One way to address this issue is to learn a single stochastic policy that visits all states as uniformly as possible, which corresponds to solving the Maximum State-Visitation Entropy (MSVE)~\citep{hazan2018provably,de2003linear,lee2019efficient} exploration problem.
In contrast to methods with vanishing bonuses that eventually stop exploring, the MSVE approach results in a stationary optimisation setup that converges to a single, stochastic exploration policy that follows different paths in different episodes, in order to cover the entire state space.

Past approaches on MSVE have been theoretically justified for discrete state spaces, as they try to maximise the discrete Shannon entropy of the state-visitation distribution~\citep{hazan2018provably,de2003linear,lee2019efficient,pong2019skew}. While some practical approaches can be empirically applied to continuous state spaces by using density estimation techniques~\citep{lee2019efficient,pong2019skew}, their theory is still focused on discrete entropy; they do not take into account the underlying geometry implied by their density estimator. To bridge this gap, we build upon the Geometry-Aware Information Theory (GAIT) framework~\citep{gallego2019gait}, which relies on a similarity function $k(x, x')$ between states to capture the underlying geometry. Then, a geometry-aware version of Shannon entropy can be defined in terms of $k$ that works across both discrete and continuous distributions. Ideally one would like to learn a similarity function that captures the specific geometry/structure of the problem domain.

To learn a geometry-aware exploration policy, one may directly try to maximise the geometry-aware Shannon entropy of the state-visitation distribution w.r.t. both the policy and similarity function $k$. However this direct approach may not succeed due to the following challenges:
\textbf{(i) Learning collapse.} The similarity function $k$ will ignore the geometry of the problem and collapse to an indicator function for the sake of increasing entropy, and thus is ill-defined for continuous domains. \textbf{(ii) Intractability.} Even with a fixed $k$, obtaining an unbiased estimate of this objective (or its gradients) is intractable in its original form.

In this paper we address these challenges by introducing \GEMFullName (\GEM), a novel exploration algorithm for learning a policy that maximises the geometry-aware Shannon entropy of state-visits for a given similarity $k$. \GEM casts the geometry-aware MSVE exploration problem as a tractable noise-contrastive estimation (NCE) problem~\citep{ContrastiveHyvarinen}, via optimising a single objective function whose unbiased gradient estimates are easily computed. Maximising the \GEM objective results in simultaneously learning both the optimal exploration policy and its corresponding state-visitation distribution. We also address the problem of collapse in learning the similarity function by adding an Adjacency Regularisation (\AR) term which allows \GEM to learn a geometrically meaningful similarity function $k$. \GEM with \AR follows two general principles: states that are close in time should be embedded closely (and be similar), whereas states that are sampled independently from the state-visitation distribution should be embedded apart (and be dissimilar). Finally, in our experiments, we show the efficiency of \GEM in solving several discrete and continuous RL problems with sparse rewards, compared against other deep RL exploration approaches, namely Random Network Distillation (RND)~\citep{burda2018exploration} and Never Give Up (NGU)~\citep{Badia2020Never}.

\section{Background}

\paragraph{Markov Decision Processes (MDPs) and Reinforcement Learning.}
MDPs model stochastic, discrete-time and finite action space control problems~\citep{bellman1965dynamic,bertsekas1995dynamic,puterman1994markov}.
An MDP is a tuple $(\X,\A,R,P,\gamma,T)$ where $\X$ is the state space, $\A$ the action space, $R$ the reward function, $\gamma\in (0,1)$ the discount factor, $T$ the length of the episode, and $P$ a stochastic kernel modelling the one-step Markovian dynamics, with $P(y|x,a)$ denoting the probability of transitioning to state $y$ by choosing action $a$ in state $x$; $P$~is also assumed to comprise a distribution for the initial state of the MDP.

A stochastic policy $\pi$ maps each state and time to a distribution over actions $\pi(\cdot|x, t)$ and gives the probability $\pi(a|x, t)$ of choosing action $a$ in state $x$ at time $t$. 
The \textit{RL objective} is to maximise the expected discounted sum of rewards: $\mathbb{E}^\pi\big[\sum_{t=1}^{T}\gamma^{t-1} r_t\big]$ where $r_t=R(x_t,a_t)$ and $\mathbb{E}^\pi$ is the expectation over the distribution of trajectories $(x_1,a_1,\dots, x_{T+1})$ from policy $\pi$.

Deep RL uses deep neural networks as function approximators~\citep{mnih2015human, mnih2016asynchronous,espeholt2018impala, lillicrap2015continuous}.
One class of such methods are \textit{policy gradient methods}~\citep{williams1992simple,espeholt2018impala}, which we build on to do MSVE exploration. In its simplest formulation, a deep policy gradient method learns a neural network policy $\pi_\theta$ with parameters $\theta$, by doing gradient ascent on the RL objective with respect to $\theta$ \citep{sutton1999policy}.

\paragraph{Geometry-Aware Shannon Entropy.}
Given a state space $\X$, we endow an underlying geometry by defining a symmetric similarity function $k \colon \X \times \X \rightarrow [0, 1]$, where $k(x, x') = 1$ means identity and $k(x, x') = 0$ means full dissimilarity. Then given a probability distribution $p(x)$ over $\X$, we define its \textit{similarity profile} as $p_k(x) \triangleq \mathbb{E}_{x' \sim p} \left[k(x, x')\right]$, which can be considered the smoothed probability of $x$. For example, if $k(x, x') \triangleq \1(\|x-x'\| < \epsilon)$, then $p_k(x)$ is the probability of a small $\epsilon-$neighbourhood around $x$, where $\1(\cdot)$ denotes the indicator function.

Then the Geometry-Aware Shannon Entropy of $p$ with similarity $k$ is defined as \citep{gallego2019gait}:
\begin{equation}
    H_k(p) \triangleq -\mathbb{E}_{x \sim p}\left[ \ln p_k(x) \right].
\end{equation}
If $\X$ is discrete, and $k(x, x') = \1(x = x')$, then the similarity profile reduces to $p_k(x) = p(x)$ and we recover the standard Shannon entropy: $H(p) = -\E_{x \sim p} \left[ \ln p(x) \right] $. However note that in general, with a suitable similarity function, the maximum geometry-aware Shannon entropy distribution can look very different from standard maximum Shannon entropy, as the similarity function is able to decide which states to cluster together and which states to be far apart, resulting in a uniform distribution over space (induced by the similarity) rather than over discrete points.

\section{{\GEM} Approach}
\label{sec:taco}

We want to solve the geometry-aware MSVE exploration problem. Specifically, we consider finding an exploration policy $\pi^\star_{E}$ that maximises the geometry-aware Shannon entropy of its stationary state-visitation distribution:
\begin{equation}
\label{eq:max.geent}
 \pi^\star_{E}\in\underset{\pi}{\mathrm{argmax}} \, H_k(p^{\pi}),
\end{equation}
where $p^\pi(x) \triangleq \frac{1}{T} \sum_{t=1}^T p^\pi_t(x)$ and $p^\pi_t(x)$ is the probability that the MDP will be in state $x$ at timestep $t$ when following $\pi$. In general, the solution to this optimisation problem is a stochastic policy that, over many episodes, ends up visiting as many different states as uniformly as possible.

To learn the optimal geometry-aware exploration policy $\pi^\star_E$ one may choose to directly optimise the objective function $H_k(p^{\pi})$ which is a well-defined and differentiable function of $\pi$. However the problem of maximising this objective is a challenging optimisation problem. At a high level, this is due to the fact that obtaining an unbiased estimate of the gradient of $H_k(p^{\pi})$ is not possible in the standard direct way due to having an expectation inside the non-linear logarithmic term $\ln\E_{x'\sim p^{\pi}}[k(x,x')]$.

A common approach to deal with this type of intractability is to use alternating optimisation, in which as a sub-routine for \cref{eq:max.geent} the learner first tries to approximate the term $\E_{x'\sim p^{\pi}}[k(x,x')]$ and then uses it to maximise an estimate of $H_k(p^{\pi})$ w.r.t. $\pi$.
These alternating methods have been used for solving the Shannon-entropy MSVE as the objective \citep{lee2019efficient,pong2019skew}. Unfortunately, the alternating approaches cannot eliminate the problem of bias in the estimate of objective function, and often are prone to instability and slow convergence \citep{bojanowski2017optimizing,goodfellow2016nips,nemirovski2004prox}. 

We tackle the issue of MSVE in a novel way, by defining a different, but closely related objective function whose solution also results in a geometric-aware MSVE policy, and whose unbiased gradient can be easily estimated from data. Proofs are in \cref{app:proofs}.

\subsection{From NCE to MSVE Exploration}

The idea behind \GEM can be traced back to Noise-Contrastive Estimation (NCE) ~\citep{ContrastiveHyvarinen}. The core idea of NCE is to learn to differentiate between two distributions $p^+$ and $p^-$. By optimising the contrastive loss, we can learn the ratio of probabilities $p^+/p^-$.

In \GEM, we extend NCE to the joint estimation and optimisation of the entropy of the state-visitation distribution. We begin our derivation of the full \GEM objective by first introducing a simplified special case corresponding to finite, discrete state spaces. For $h : \X\times \X \rightarrow (0, \infty)$, consider the following objective:
\begin{align}
\text{\GEM}(h, \pi) &\triangleq  \mathbb{E}_{x\sim p^{\pi}} \left[\ln(h(x,x)) \right] \nonumber\\ &\quad 
- \mathbb{E}_{x,x' \sim p^{\pi}} \left[ h(x,x') \right]+1.
\label{eq:taco.discrete.gen}
\end{align}
 Maximising this objective function can be seen as contrasting positive pairs $(x, x)$ from negative pairs $(x, x^\prime)$, where $x, x^\prime$ are i.i.d.~samples from the current state-visitation distribution. The positive term $\mathbb{E}_{x\sim p^{\pi}}[\ln( h(x,x)) ]$ tries to increase $h(x, x)$ while the negative term $\mathbb{E}_{x,x' \sim p^{\pi}} [h(x,x')]$ decreases $h(x, x^\prime)$. The key property of the objective function~\cref{eq:taco.discrete.gen}, which makes it distinct from other contrastive objective functions such as~\cite{ozair2019wasserstein, ContrastiveHyvarinen,wu2018unsupervised}, is its relation to the MSVE Shannon entropy:
\begin{restatable}{proposition}{tacodisc}{}
\label{thm:taco.disc.gen}
Given a discrete set of states $\X$ and a function $h : \X\times\X \rightarrow [0, \infty)$, we have $\max_{h,\pi} \mathrm{\GEM}(h, \pi) =  H (p^{\pi^\star_E}).$ The maximiser $h^{\star}(x,x')=\mathbf 1(x=x')/p^{\pi^{\star}_E}(x)$ when $p^{\pi^\star_E}(x) > 0$, and $\pi^\star_{E}$ is the Shannon MSVE policy.
\end{restatable}

As a result, maximising the \GEM objective over $\pi$ and $h$ simultaneously learns both the optimal state-visitation distribution and the MSVE policy as part of a single optimisation procedure.

This simple variant of  \GEM can be seen as a special case of M-estimation of KL-divergence  through Legendre–Fenchel transformation (i.e., convex conjugate) \citep{nguyen2010estimating}, by using the following equivalency between KL-divergence and Shannon Entropy:

\begin{equation*}
H(p^{\pi})=KL(Q||R),
\end{equation*}
where $Q(x,x')\triangleq \mathbf{1}(x'=x)p^\pi(x)$ and $R(x,x')\triangleq p^\pi(x) p^\pi(x')$. Thus the theoretical results of \citep{nguyen2010estimating} can be used directly to prove \cref{thm:taco.disc.gen}.

To make progress towards our general geometry-aware objective, we note that the optimal $h^\star(x, x^\prime)$ vanishes for $x\neq x^\prime$. Therefore we introduce a new function $g:\X\rightarrow(0, \infty)$ and re-parameterise $h(x, x^\prime) = \1(x=x^\prime)g(x)$. The objective function of Eq. \ref{eq:taco.discrete.gen} can then be expressed in terms of $g$ as
\begin{align}
\text{\GEM}(g, \pi) &\triangleq  \mathbb{E}_{x\sim p^{\pi}} \left[ \ln(g(x)) \right] 
\nonumber\\ &\quad 
-  \mathbb{E}_{x,x' \sim p^{\pi}} \left[\mathbf 1(x=x') g(x)  \right]+1,
\label{eq:taco.discrete}
\end{align}
and we can adapt the result from \cref{thm:taco.disc.gen} in following corollary:
\begin{restatable}{corollary}{tacodisccor}{}
\label{thm:taco.disc}
The maximiser $g^\star(x)$ for $\max_{\pi,g} \mathrm{\GEM}(g, \pi) =  H (p^{\pi^\star_E})$ is $1/p^{\pi^\star_E}(x)$ for $p^{\pi^\star_E}(x) > 0$.
\end{restatable}
In the next section, we extend \cref{eq:taco.discrete} to handle the general case of geometry-aware Shannon entropy.

\subsection{{\GEM} with Similarity Functions}
\label{sec:taco.obj}

The formulation of \GEM  in  \cref{eq:taco.discrete} is not, as-is, suitable for continuous state-spaces. The indicator function in the term $\mathbb{E}_{x,x' \sim p^{\pi}} \left[{\bf 1}(x=x') g(x) \right]$ becomes ill-defined and the objective breaks down. To generalise \GEM to continuous state-spaces, we replace the indicator function ${\bf 1}(x=x')$ with a symmetric similarity function $k(x, x')$ satisfying $k(x,x) = 1$. The state-visitation distribution $p^\pi(x) = \mathbb{E}_{x'} [{\bf 1}(x=x')]$ can then be replaced by the similarity profile $  p^\pi_k(x)\triangleq \mathbb{E}_{x'\sim p^\pi}[ k(x,x')]$, with induced geometry-aware Shannon entropy $H_{k}(p^\pi)$ \cite{gallego2019gait}. The similarity function $k$ will ideally capture the inherent structure of the state space, leading to a geometrically meaningful entropy. This can result in more efficient exploration with a suitable similarity function that ends up clustering together less important states, and thus a maximum geometry-aware policy would not visit those less important states as often as other, more important states.

We now introduce the full, geometry-aware form of the \GEM objective function:
\begin{align}
\text{\GEM}_{ k}(g, \pi) &\triangleq  \mathbb{E}_{x\sim p^{\pi}} \left[ \ln(g(x)) \right]
\nonumber\\ &\quad -\mathbb{E}_{x,x' \sim p^{\pi}} \left[k(x, x') g(x) \right] + 1.\label{eq:taco.kernel.obj}
\end{align}
We generalise \cref{thm:taco.disc} to encompass similarity functions in \cref{thm:taco.kernel}.

\begin{restatable}{theorem}{tacokernel}{}
\label{thm:taco.kernel}
Let  $g : \X \rightarrow (0, \infty)$. Given a similarity function $k \colon \X \times \X \to [0, 1]$.
Then we have $\max_{g,\pi} \mathrm{\GEM}_{k}(g, \pi) = H_{k} (p^{\pi^\star_E}),$ 
where the maximiser is  $g^\star(x) = 1/p_k^{\pi^\star_E}(x)$.
\end{restatable}

By replacing the indicator function with the similarity function $k$, the \GEM objective generalises from discrete Shannon entropy to geometry-aware Shannon entropy, and the maximiser $g^\star$ generalises from the inverse probability to the inverse similarity profile. \GEM can now be readily applied to both discrete and continuous state spaces.

Compared to the intractable problem of directly maximising for the geometric-aware MSVE policy, an unbiased estimate of the gradients of \GEM is easily computable, and thus \GEM can be solved efficiently with standard optimisers.   The following theoretical result  formalises this argument for the case that we parameterise $\pi$ and $g$ with some function approximatiors.

\begin{restatable}{proposition}{gradprop}{}
\label{prop:aco.grad}
 Let $\pi$ and $g$ be approximated by some differentiable function approximators with sets of parameters $\theta$ and $\xi$ respectively. 
Let $r^{\mathrm{\GEM}}_t\triangleq\ln(g(x_t))-[k(x_t,x'_t)(g(x_t)+g(x'_t))]$, where $x'_t$ is drawn independently  from $p^\pi$ at every time step $t$.  Then unbiased estimates of the gradient of \GEM objective w.r.t.~$\theta$ and $\xi$ are, respectively,
\begin{align}
 &\sum_{t=1}^{T-1}\nabla_{\theta}\ln(\pi(a_t|x_t))\sum_{\tau=t+1}^T r^{\GEM}_\tau,\label{eq:grad.theta}
\\
\mbox{and } \quad & \sum_{t=1}^T\nabla_{\xi} \big[ \ln(g(x_t))- g(x_t)k(x_t,x'_t)\big].\label{eq:grad.si}
\end{align}
\end{restatable}

There are a few notable remarks concerning this result.  {\bf (i)} The gradient of the \GEM  objective function w.r.t. the parameters of policy is expressed as a standard policy gradient, and so one can use a conventional policy gradient solver to efficiently optimise the $\GEM$ policy with the reward $r^{\mathrm{\GEM}}_t$. In our experiments, we use a standard V-Trace Actor-Critic approach \cite{espeholt2018impala}. {\bf (ii)}  As these two gradient terms correspond to the same objective function, one can simultaneously optimise  $\pi$ and $g$ together through a single optimisation procedure, avoiding the need for alternating optimisation techniques. The fact that the geometry-aware MSVE problem can be solved efficiently in this way may seem surprising, as its solution is the same as the intractable optimization problem of \cref{eq:max.geent}. The difference is that \GEM casts the problem of geometry-aware MSVE as a joint optimisation in terms of the exploration policy $\pi$ as well as $g$. In this larger function space the problem of geometric-aware MSVE is no longer intractable and can be solved efficiently, unlike the optimization problem of \cref{eq:max.geent} in which only $\pi$ is optimised. {\bf (iii)} The main \GEM objective (\cref{eq:taco.kernel.obj}) can be further generalised to maximise a geometry-aware version of Tsallis entropy. See \cref{app:proofs} for details.

\subsubsection{Learning a Similarity Function}
\label{sec:similarity}

So far, we have assumed that a similarity function $k(x, x')$ has been given. However, we would ideally like to learn a geometrically meaningful $k$ from data as a part of \GEM. One can simply learn the similarity function by maximising \GEM w.r.t. $k$ as well. If we do this in the finite, discrete case, we can recover \cref{thm:taco.disc} from \cref{thm:taco.kernel} as shown in the following proposition:

\begin{restatable}{proposition}{propdisc}{}
\label{prop:discrete}
If $\X$ is a finite set, then $\max_{k, g, \pi} \mathrm{\GEM}_{k}(g, \pi) = H (p^{\pi^\star_E})$ is attained for $ k^\star(x, x') = \mathbf{1}(x = x')$ and $ g^\star(x,\pi) = 1/p^{\pi^\star_E}(x)$.
\end{restatable}

Note that $k$ converges to an indicator function. However, in the continuous case, trying to converge to the indicator function makes the whole problem ill-defined. To show this effect, we use \GEM to learn a simple 1D bi-modal Gaussian mixture. We parameterise $g$ with a small MLP. We consider two versions of $k$: (i) one version is fixed with $k(x, x') = e^{-2|x - x'|}$ which represents what we think a good similarity function should be, and (ii) we parameterise $k$ using a neural network embedding function $f$ as $k(x, x') = e^{-\| f(x) - f(x')\|_2}$ and maximise over it as well. For the distributions to be learned, we consider both a discretised and the continuous version of the simple 1D bi-modal Gaussian mixture.

\begin{figure}
\centering
\includegraphics[width=0.5\textwidth]{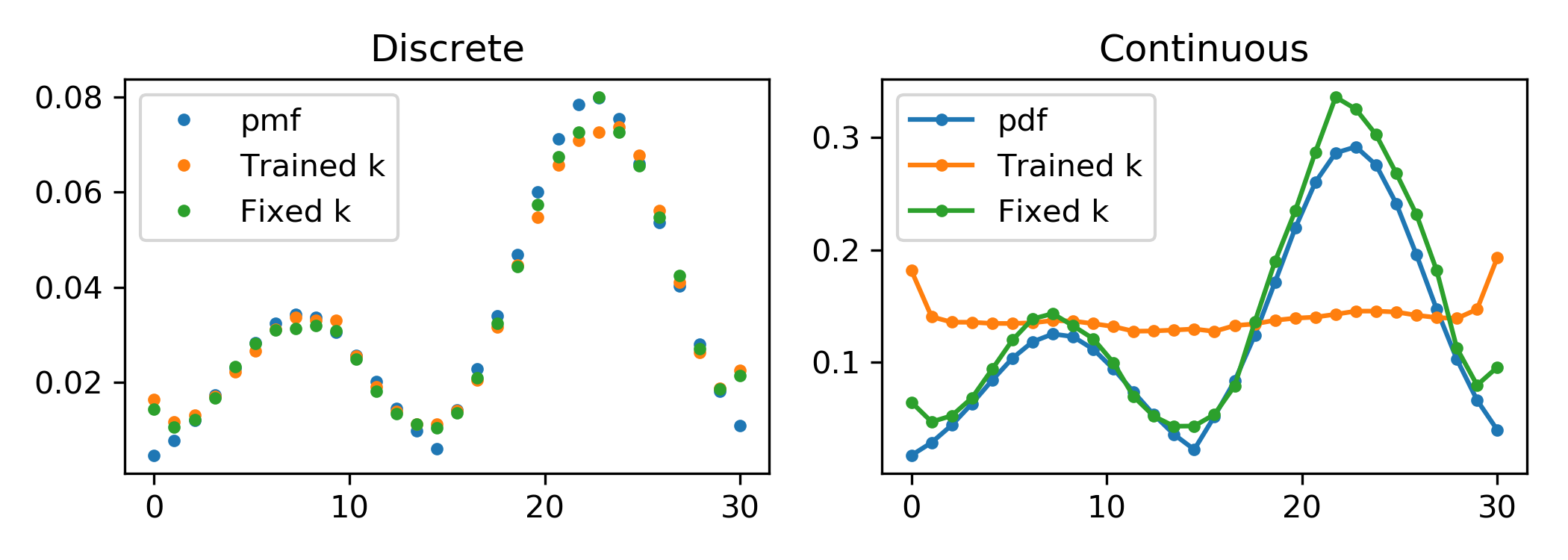}
\caption{Learned distributions with \GEM. }
\label{fig:tacogaussian}
\end{figure}

\cref{fig:tacogaussian} shows the learned probabilities and densities computed from $g$ for the four configurations. For the discrete distribution shown on the left, both a fixed and a learned similarity function perform well and can approximate the probability distribution. For the continuous case on the right, the fixed similarity learns a close approximation, while the learned similarity collapses into a uniform distribution. The reason for the collapse is that the non-linear embedding function $f$ is able to flatten out the data distribution, converting the original distribution into a uniform distribution; which is the maximum entropy distribution.

Therefore, in all but the simplest cases, we cannot solely learn a similarity function through maximising the \GEM~objective. In general, we may need to rely on domain-specific knowledge in order to learn a well-structured and meaningful similarity function. Nonetheless for MDPs, it is possible to learn a generic yet meaningful similarity based on \textit{temporal structure}. We keep the parameterisation of the similarity function as $k(x, x') =  e^{-c\| f(x) - f(x')\|_2}$, and augment our objective with an additional regularization term on the embeddings $f$, which we call \emph{adjacency regularisation}. The \emph{adjacency regularisation} (\AR) is a generalised pseudo-Huber loss with exponent $q$ and offset $\delta$:
\begin{equation}
    \label{eq:ar.objective}
    \AR \triangleq -\sum_{t=1}^{T} \left(\delta^q + \|f(x_t) - f(x_{t+1})\|_2^q \right)^{1/q},
\end{equation}
where $x_t \sim p^\pi_t$, $a_t \sim \pi(x_t)$ and $x_{t+1} \sim p(x_t, a_t)$.
This regularisation pulls together the embeddings of time-adjacent states. Combined with \GEM, which tries to expand the distance between embeddings, this trains an embedding where approximately the closest states are those that are 1-step apart; the second closest states are those that are 2-steps apart; and so on and so forth. This is similar to the notion of reachability~\citep{savinov2018episodic}, which tries to learn the expected shortest number of steps between any two states. We note that the higher values of exponent $q$ are more robust to distribution shift and relate closer to a reachability distance. In our experiments we find than that $q=4$ works well.

Note that \AR~is symmetric in time. This means that \AR ends up averaging the contribution of asymmetric transitions. For example, a one-way transition between $x_t$ and $x_{t+1}$ is approximately equivalent to a symmetric transition with probability $0.5$.

\subsection{{\GEM} Algorithm}
\label{sec:algorithm}

The final \GEM algorithm is the combination of the geometry-aware \GEM objective, \AR, and maximising the environment reward (see \cref{alg:tacomax}). To train the embedding $f$ we optimise the combination of the \AR objective and \GEM objective (\cref{alg:f}). To train $g$, we maximise the empirical \GEM objective (\cref{alg:g}).\footnote{Note that $r(x_t)$ is a constant so its gradient  with respect to $g$ or $f$ equals zero.}
Finally, to train $\pi$, we use a policy gradient algorithm (\cref{alg:pg}). Note that $\pi$ is optimised with respect to the empirical estimate of the \GEM objective function plus the environment reward, aimed at striking a balance between MSVE exploration and exploitation. More details on the algorithm is in \cref{app:expdetails}.

\begin{algorithm}[ht]
 \KwInputs{neural networks $g$ and $f$, policy $\pi$, Huber offset $\delta$,
 Huber exponent $q$, regularisation scale $c$}
 \For{$i\leftarrow 1$ \KwTo $\infty$}{
    Sample batch of episodes $B = \{(x_t, r_t)\}$ \\
    Sample batch of episodes $B' = \{x'_t\}$ \\
    For every $x_t$ draw a randomly shuffled  sample  $x'$ uniformly form  $B'$ \\
    Set $k(x_t,x') = e^{-\| f(x_t) - f(x')\|_2}$ \\
    Let $r^{\mathrm{\GEM}}_t = \ln g(x_t) - k(x_t,x')\left(g(x_t)+g(x')\right)$ \label{alg:r} \\
    Set $R_{t} = r_t + r^{\mathrm{\GEM}}_t$ \\
    Take gradient step for $\underset{g}{\max} \; \underset{t}{\sum} R_{t} $ from \cref{eq:grad.si} \label{alg:g} \\
    PolicyGradientStep for $\underset{\pi}{\max
    }\; \underset{t}{\sum} R_{t} $ from \cref{eq:grad.theta} \label{alg:pg} \\
    Take gradient step for $\underset{f}{\max} \; \underset{t}{\sum} R_t - c \cdot \left(\delta^q + \|f(x_t) - f(x_{t+1})\|_2^q \right)^{1/q}$ \label{alg:f}
 }
 \caption{Pseudocode of \GEM}
 \label{alg:tacomax}
\end{algorithm}

\section{{\GEM} Experiments}
\label{sec:experiments}

In this section, we conduct in-depth experimental analyses of various aspects of the \GEM~algorithm. All plots show 95\% confidence intervals over 5 seeds. The details on the implementation and the choice of hyper-parameters are in \cref{app:expdetails}.

\subsection{Intrinsic Reward Normalization}

The magnitude and variance of the intrinsic reward can vary greatly across different environments.
To mix intrinsic and extrinsic rewards in a consistent way in our experiments, we first standardise, scale, and shift our intrinsic reward $r_t^{\GEM}$:
$r_t^{\GEM} \rightarrow \left( \frac{r_t^{\GEM} - \mu}{\sigma}\, s + m\right),$
where $\mu$ (respectively $\sigma$) is an exponential running average of the mean (respectively standard deviation) of the intrinsic reward, and $s$ and $m$ are hyperparameters for the new standard deviation and mean. We then sum this normalized intrinsic reward with the extrinsic reward $R_t = r_t + r_t^{\GEM}$. We found that using this more general normalisation scheme allows \GEM~to be robust to different resolutions of the learned embedding and different sizes of environments.

\subsection{Comparisons with Baselines}

To test the basic exploration capabilities of \GEM, we start with a simple \texttt{2-Rooms} gridworld environment consisting of two rooms and a bottleneck connection (\Cref{fig:envs}). The agent starts in the top-left-hand corner, and thus must go all the way to the right and then down to pass the bottleneck in order to reach the goal in the bottom-left corner. 

To test the ability to do exhaustive exploration, we designed a tree-like gridworld: \texttt{16-Leaves} (\Cref{fig:envs}). The agent starts near the centre, and needs to branch off in multiple directions, ending up in one of the $16$ ``leaves''. Every episode, one of the green squares in a leaf is randomly picked, and a reward block is placed there; all the other green squares become normal free squares. Moreover, each episode is barely long enough for an agent to reach the end of one leaf. Therefore, an agent must exhaustively explore all leaves  in order to learn that visiting the green square results in a reward. Importantly, the agent must be able to visit different leaves across episodes. In both the \texttt{2-Rooms} and \texttt{16-Leaves} environments, the agent is given as input an RGB observation from which it must learn to extract relevant features with a deep convolutional neural network.

\begin{figure}[th]
    \centering
    \includegraphics[width=0.9\linewidth]{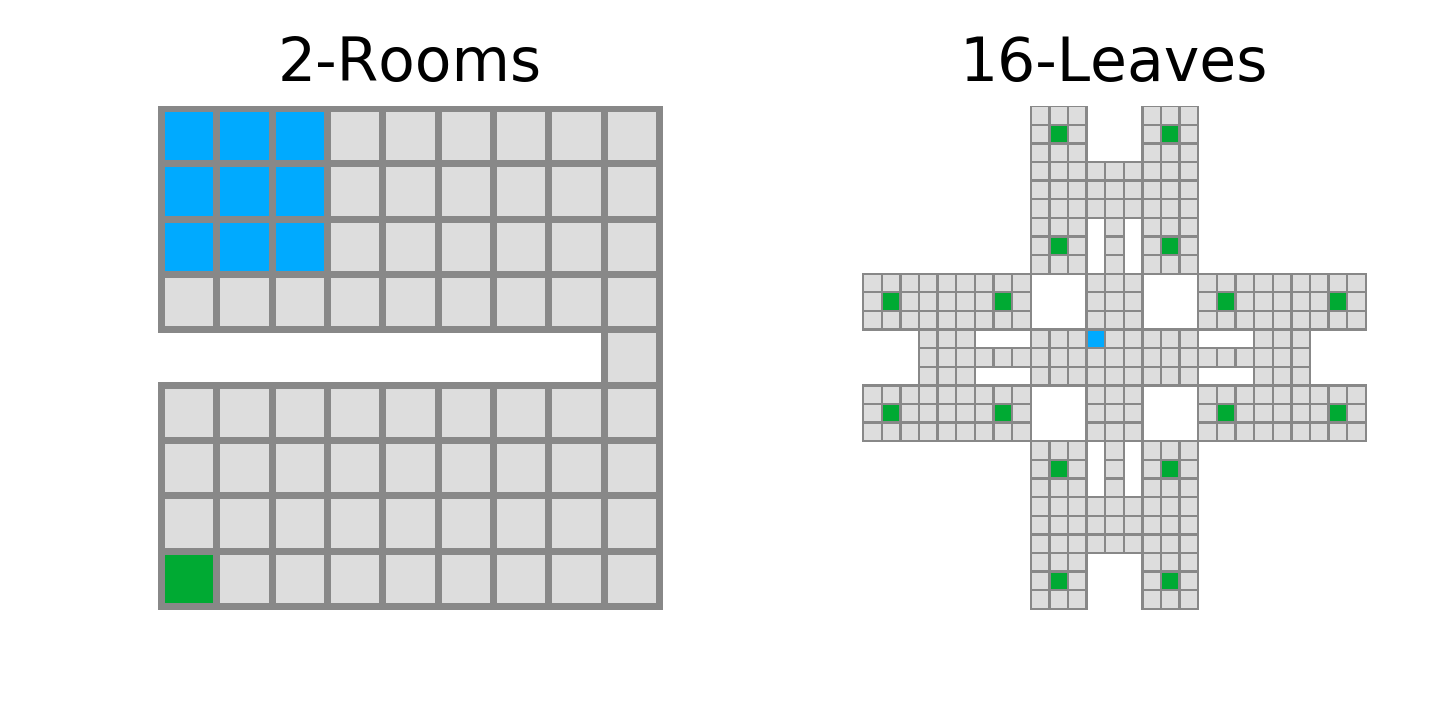}
    \caption{\texttt{2-Rooms} and \texttt{16-Leaves} gridworld environments. Every episode, the agent spawns randomly in one blue square, and a reward spawns randomly in one of the green squares.}
    \label{fig:envs}
\end{figure}

\subsubsection{Comparison with the Empirical Oracle MSVE}
\label{sec:sweep2timescale}

In this experiment we compare the performance of \GEM with an idealised variant of alternating optimisation approach taken by some prior work \citep{lee2019efficient, pong2019skew}. The basic alternating optimisation approach is to first estimate the stationary distribution over states $p^\pi(x)$, often using a generative model or a non-parametric estimate such as kernel-density estimation. Then the intrinsic reward is set to be $-\log p(x)$, in order to drive the policy towards less visited states.

We use an empirical oracle to estimate the stationary distribution $p^\pi(x)$, by using an exponential moving average of visitation counts of the true states, which is essentially the best that can be done for these simple, discrete domains. Note that this uses \emph{privileged information} on the state index to get access to the true state-visitation counts from the environment. Thus, this is not a practically implementable approach in larger or continuous domains. Also note that no geometry-aware information is being used; standard Shannon entropy is being maximised over the true states.

We simulate an alternating optimisation scheme that updates the counts every step, but only updates the policy every $n$ steps, where $n \in \{1, 5, 10\}$. This simulates the effect of letting the inner optimization take more gradient steps in order to be more accurate for the case when using function approximation. Because we use true counts, our alternating optimisation is stable even for $n=1$, but in general with function approximation, $n>1$ may be necessary for stability.

\Cref{fig:sweep2timescale} compares these methods with \GEM. \GEM~is just as fast as the idealised case of $n=1$, and multiple times faster than larger values of $n$.
These results show that the \GEM~algorithm is very efficient and confirms that it does not require complex optimization schemes to be stable; everything is optimized and updated together. Furthermore, the reason \GEM is able to match the oracle even with function approximation is because \GEM takes advantage of geometry information, which allows \GEM to focus exploration on more interesting states. See \cref{sec:adjreg} for a more in-depth look at how geometry affects \GEM exploration.

\begin{figure}
    \centering
    \includegraphics[width=1.0\linewidth]{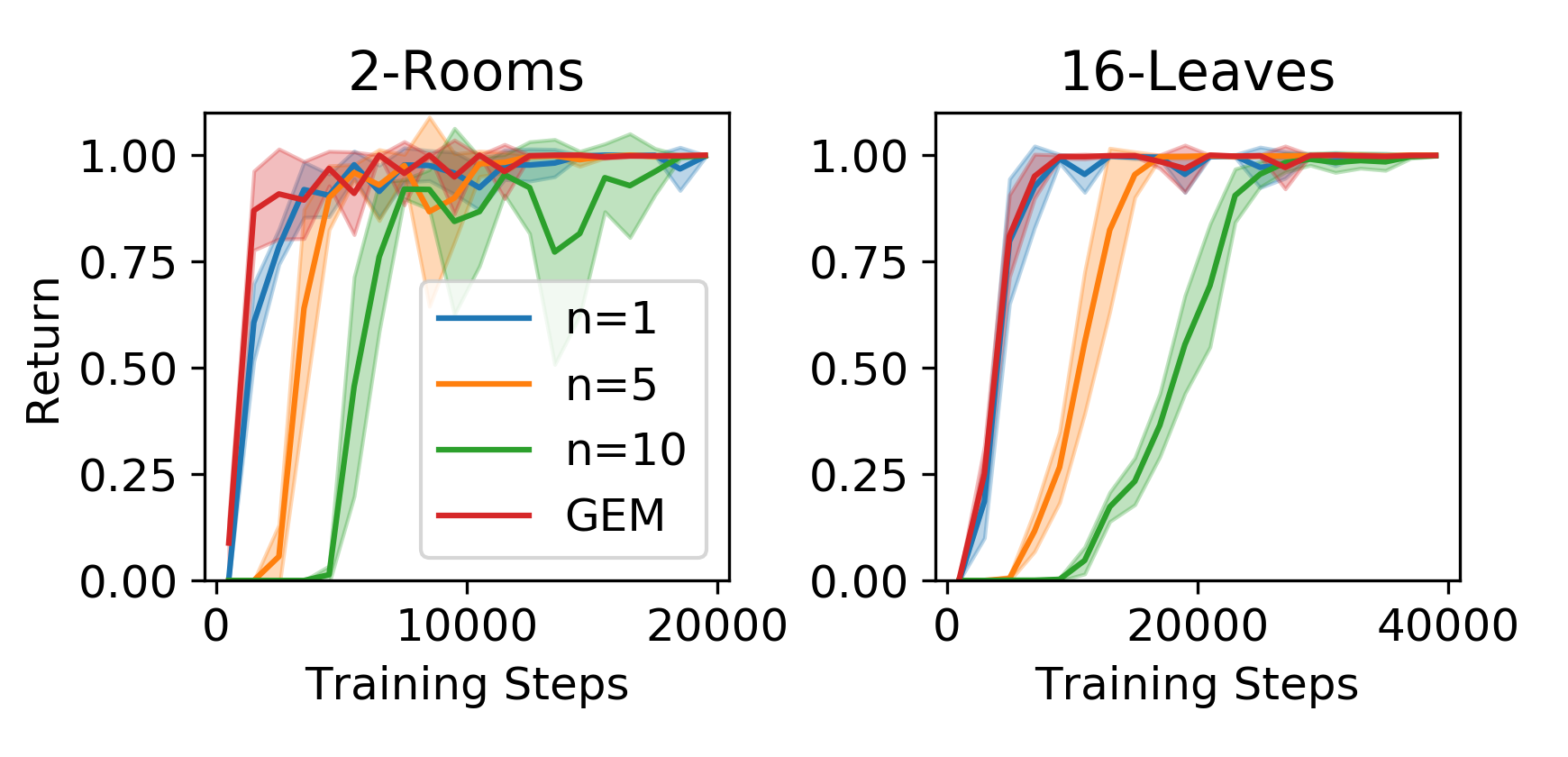}
    \caption{Empirical-Oracle MSVE vs. \GEM}
    \label{fig:sweep2timescale}
\end{figure}

\subsubsection{Comparison with NGU and RND}
\label{sec:ngurnd}

Here, we compare against 3 baselines: RND, NGU without RND, and NGU with RND. RND (Random Network Distillation) is a simple technique that computes an intrinsic reward based on prediction error of predicting a fixed, random projection of the observation \citep{burda2018exploration}. It is an example of a transient intrinsic reward that eventually vanishes.
NGU (Never Give Up) keeps an episodic memory of visited state embeddings, and computes an intrinsic reward for the next state that is inversely correlated with its distance to previously visited states \citep{Badia2020Never}. NGU converges to a policy that tries to visit diverse states within an episode, but has no incentive to be diverse across episodes. The combination of NGU with RND uses NGU to be diverse within episodes, and relies on RND to be diverse across episodes and across training, and is state-of-the-art for exploration on Atari \citep{badia2020agent57}.

\Cref{fig:comparengurnd} shows the comparison of the baselines with \GEM.
Along with comparing on our gridworlds, we also consider two additional continuous-state domains, \texttt{MountainCar} and \texttt{CartpoleSwingup}, to illustrate the effectiveness of \GEM in continuous domains. 

The behaviour of the \GEM~algorithm is similar to NGU with RND, as trying to maximize entropy implies diversity both within an episode and across episodes in order to visit all states. However \GEM~is able to achieve this kind of behaviour through a single, simple, principled objective, without needing to add extra components.

From \Cref{fig:comparengurnd} we see the shortcomings of NGU without RND, which is not able to be diverse across episodes, and thus is not able to efficiently explore all the different leaves of \texttt{16-Leaves}, and converges to a sub-optimal policy. RND is also not very effective, as it is slow and unstable, and heavily dependent on neural network architecture. NGU with RND is able to combine the best of both and consistently solves all tasks. Finally, \GEM~is also able to achieve the best of both worlds and be faster at solving the gridworld tasks.

\texttt{MountainCar} and \texttt{CartpoleSwingup} domains are standard, sparse-reward environments. Both are continuous-state domains with asymmetric transitions. We see that although the \AR used in the \GEM objective is symmetric, \GEM has no problem learning a useful state embedding and similarity function for solving these tasks, and is comparable with baselines.

\begin{figure}[th]
    \centering
    \includegraphics[width=1\linewidth]{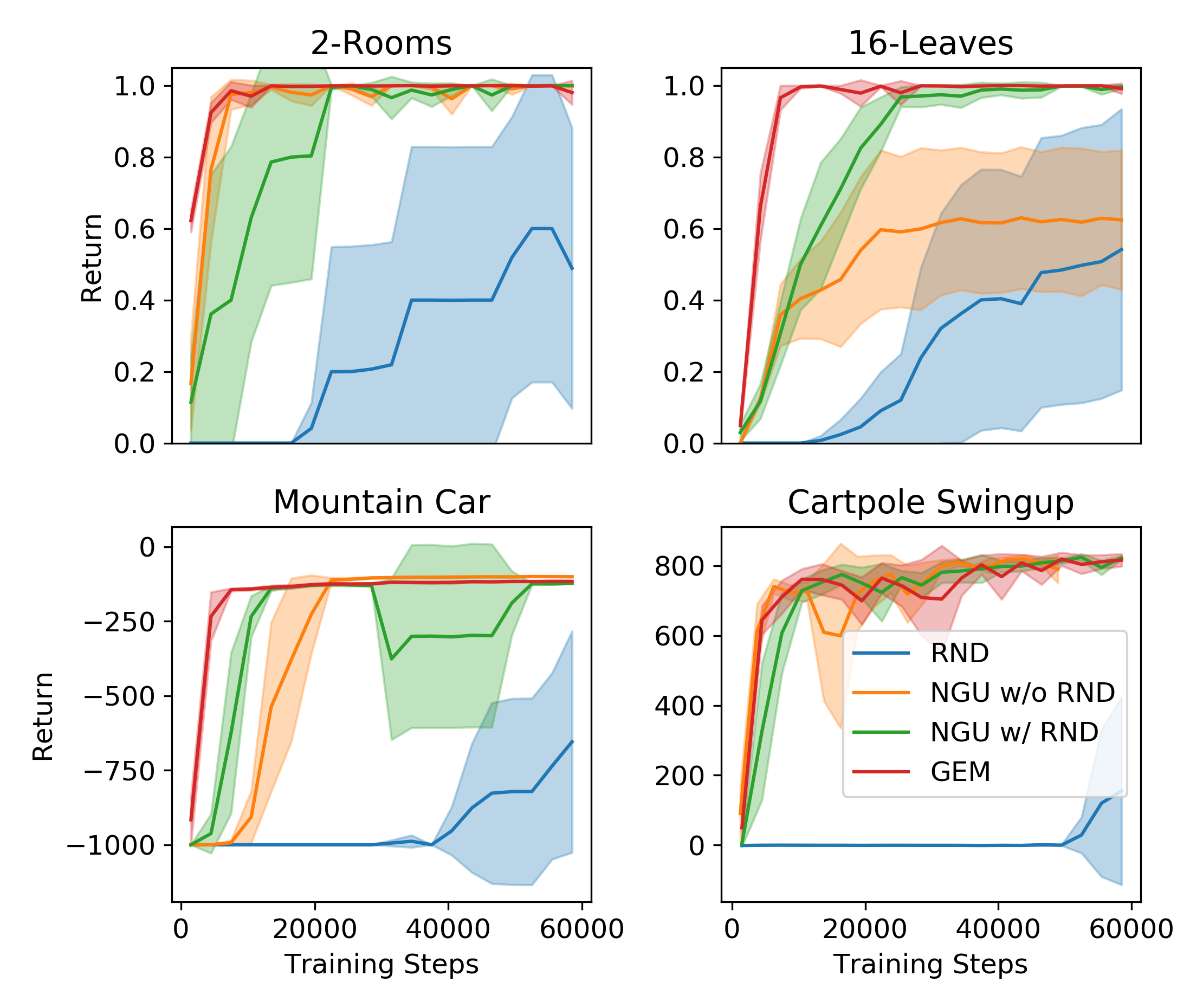}
    \caption{Comparison of \GEM versus baselines across multiple environments.}
    \label{fig:comparengurnd}
\end{figure}

\subsection{Further Analyses}
\subsubsection{Impact of Adjacency Regularization}
\label{sec:adjreg}
In this section, we show the effect of using adjacency regularization (\AR) to shape the embeddings and similarity function. We compare \GEM~with and without \AR on the \texttt{2-Rooms} environment. We also test on a noisy version of the environment, where we add a square whose colour is randomly chosen from $65536$ different colours at every step.

\begin{figure}[th]
    \centering
    \includegraphics[width=0.9\linewidth]{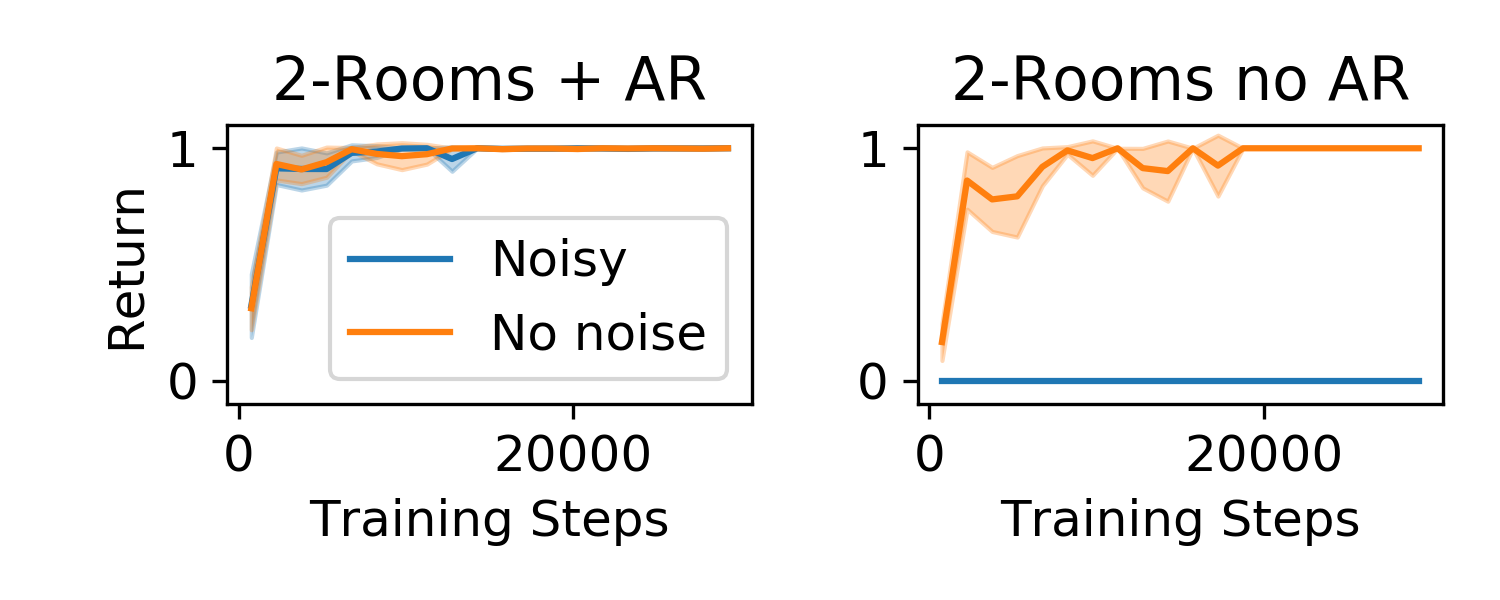}
    \caption{\GEM (left) and \GEM without \AR (right) in noisy and noiseless \texttt{2-Rooms}.}
    \label{fig:togglenoiseadj}
\end{figure}

In \Cref{fig:togglenoiseadj}, we see that when using \AR, there is no difference in performance between noiseless and noisy environments. On the other hand, without \AR, the agent completely fails to solve the noisy task. In fact, it fails improve upon a uniformly random policy.

\begin{figure}
    \centering
    \includegraphics[width=0.8\linewidth]{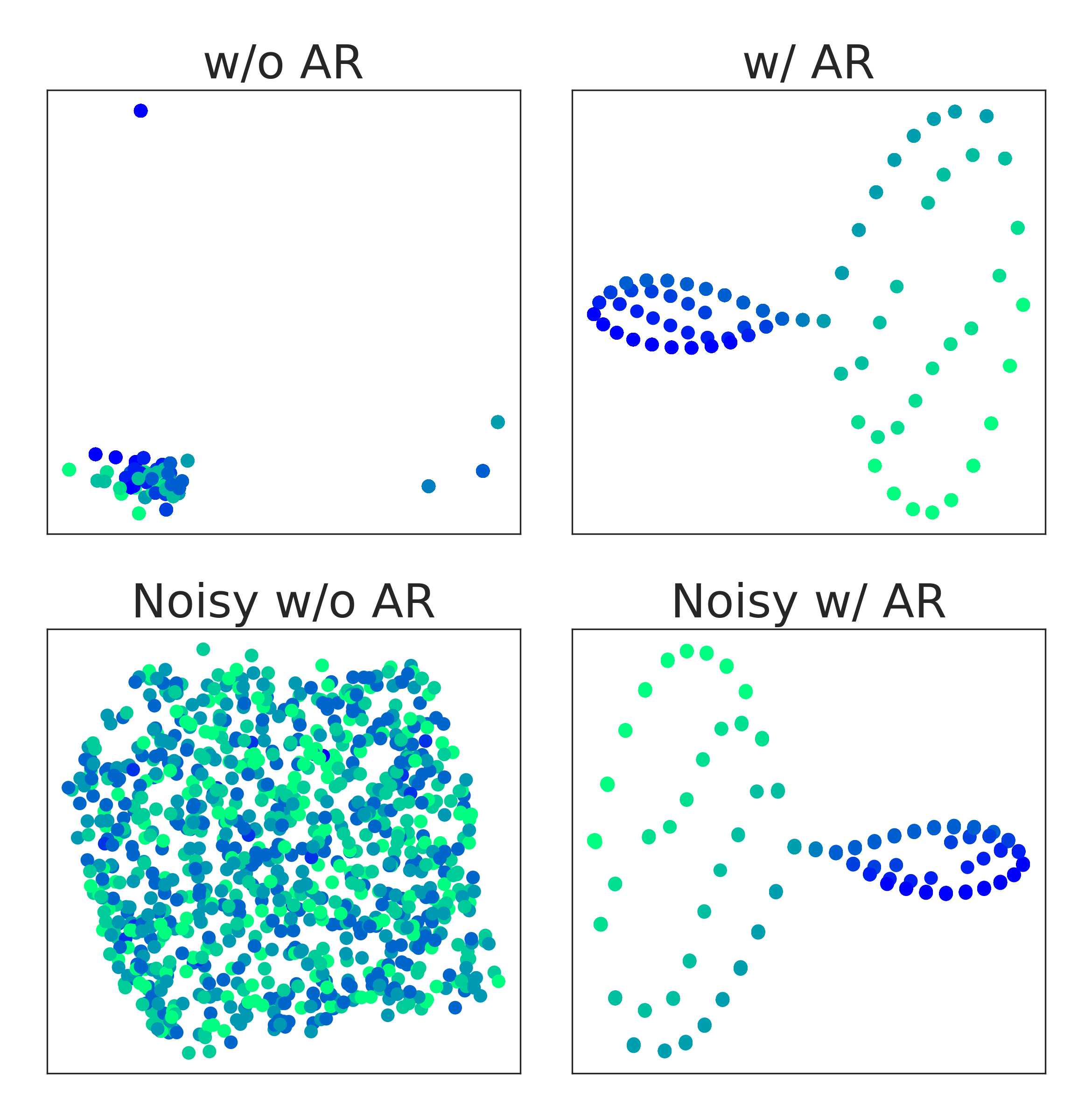}
    \caption{\texttt{2-Rooms} Embeddings for different settings of \AR and noise.}
    \label{fig:2roomembedding}
\end{figure}

Taking a closer look, we examine the embeddings learned for the 4 different settings. \Cref{fig:2roomembedding} shows the first two principal components of the learned embeddings. We see that \AR is able to learn the grid-like geometry of \texttt{2-Rooms} and is completely robust to noise. However, without \AR, the learned embeddings are completely random, and are heavily impacted by the noise instead of ignoring it. Since the noise is a much larger source of entropy than the agent position, \GEM~can trivially maximize entropy by just paying attention to the noise. Further insights into the structure of learned embeddings can be found in \cref{app:additional_experiments}.

\subsubsection{Embedding Resolution}
\label{sec:sweepsim}

In this section, we examine how the exploration behavior of \GEM is affected by the resolution of the embeddings and the similarity function. We can make the resolution \emph{finer} by making the distance between embedding points \emph{larger}, and thus less similar (a coarser resolution means distances are smaller, and thus states are more similar). This can be done by increasing the Huber offset $\delta$ in \AR (\cref{eq:ar.objective}), or decreasing the scale of the regularisation $c$ in the total loss.

\begin{figure}[th]
     \centering
     \begin{subfigure}{0.48\linewidth}
         \centering
         \includegraphics[width=\linewidth]{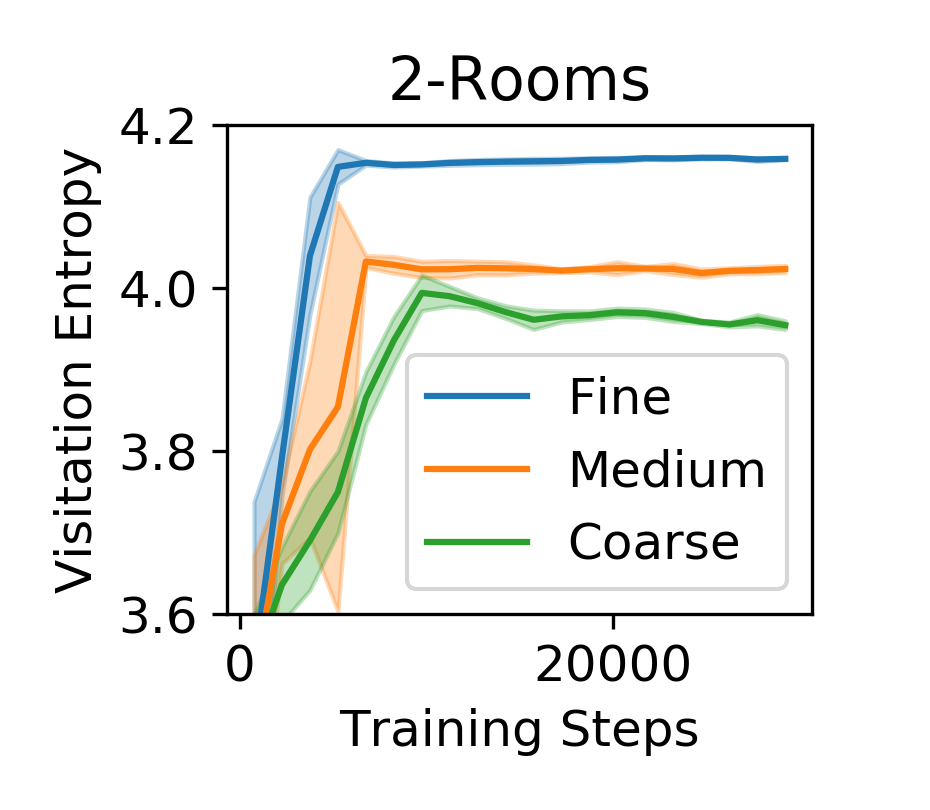}
     \end{subfigure}
     \begin{subfigure}{0.5\linewidth}
         \centering
         \includegraphics[width=\linewidth]{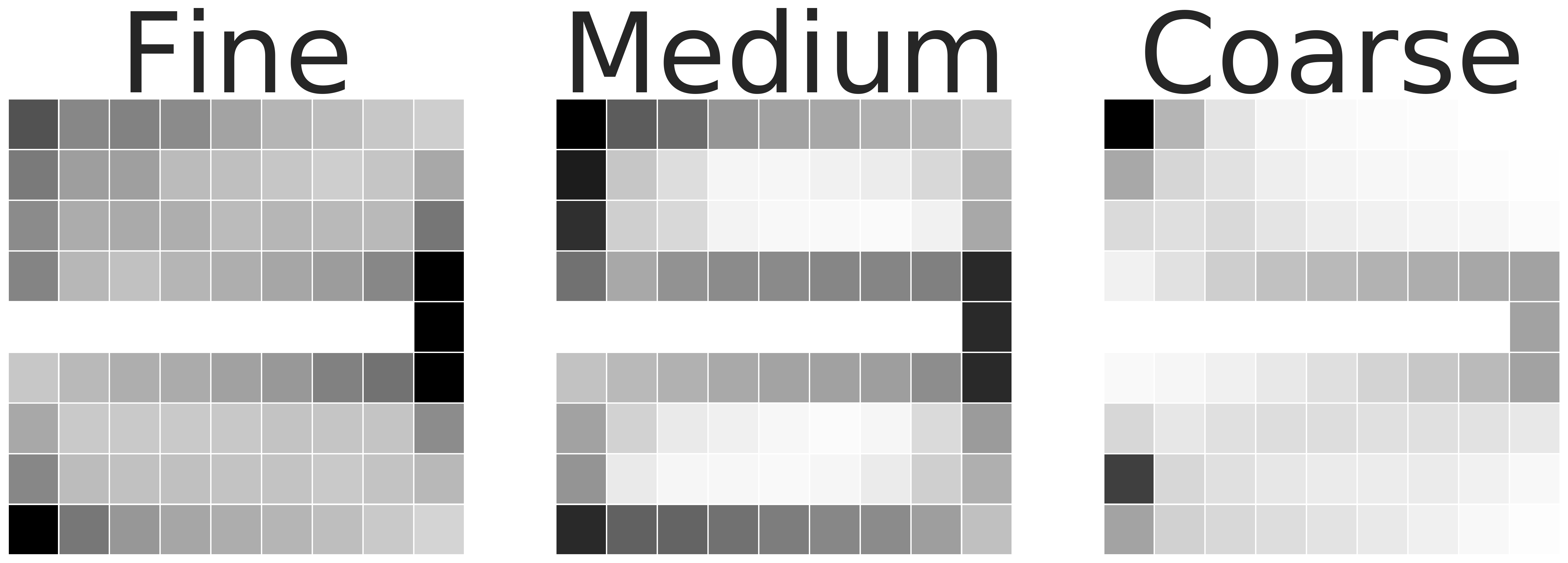}
     \end{subfigure}
     \caption{Visitation entropies (left) and heatmaps (right) for different embedding resolutions.}
     \label{fig:sweepdist}
\end{figure}

We compare three different resolutions: coarse ($\delta=0.3, c=20$), medium ($\delta=0.6, c=10$), and fine ($\delta=1, c=1$).
\Cref{fig:sweepdist} shows the true visitation entropy achieved for these resolutions, with finer resolutions resulting in larger entropy. This is because coarser embeddings cluster states closer together, leading to a smaller effective state space and a smaller maximum visitation entropy. A different illustration of this phenomenon is given in \Cref{fig:sweepdist}, which shows the visitation heatmaps in \texttt{2-Rooms} for the three resolutions. For the fine case, the state-visitation distribution is close to uniform in each of the rooms, whereas for the medium case it becomes skewed towards the room boundaries. This is due to states in the middle of the rooms being clustered together, while the opposite boundaries of the room are dissimilar enough that they count as distinct states. Finally, for the the coarse case, the embeddings are so close together that they effectively form a line and there is no entropic incentive to deviate from a single path.

These results show that the exploratory behavior can be modified by tuning the resolution of the embeddings. In particular, the ability to reduce the effective state space size by employing coarser embeddings is especially important when scaling up the algorithm to larger and larger state spaces.

\section{Related Work}

Prior work in MSVE exploration used an iterative algorithm that requires retaining a mixture of all past policies as well as kernel density estimation for estimating densities~\citep{hazan2018provably}, which is hard to scale to larger and complex domains for which a kernel is hard to define. Approximations were later which improved computational complexity but only for the tabular setting~\citep{mutti2020intrinsically}. Maximum entropy can also be viewed as trying to match the state-visitation distribution to a uniform distribution, as a special case from~\citep{lee2019efficient}, which casts the problem as a two-player minimax game. Thus this approach requires the use of alternating optimisation schemes which need additional optimisation tweaks to be stable. It is also possible to estimate differential Shannon entropy using non-parametric k-nearest neighbour methods~\citep{beirlant1997nonparametric}, however their connection to geometry-aware Shannon entropy is unclear. Beyond geometry-aware Shannon entropy, there is work on other entropy measures that are geometrically-aware such as Sinkhorn negentropy \citep{mensch2019geometric}.

Maximising diversity instead of entropy is a similar objective that tries to visit as many different states as possible. Approaches such has maximising diversity between trajectories can be used to discover skills for reaching diverse sets of states~\citep{lim2012autonomous,gajane2019autonomous,gregor2016variational}, but lack a way to inject structural information. One can also maximise diversity within a single trajectory~\citep{Badia2020Never}, but it relies on costly non-parametric estimates. As a step towards maximum entropy, \citep{islam2019marginalized} optimises a variational lower bound as a form of regularisation. Furthermore, maximising the diversity of goal states in goal-conditioned RL~\citep{andrychowicz2017hindsight,pong2019skew} is another way of performing exploration based on the principle of maximum entropy, but they are not equipped to leverage on the geometry of problem domain.  They also require the use of alternating optimisation schemes as well.

Other approaches that use intrinsic rewards for exploration design such exploration bonuses to compute a non-stationary, decaying novelty~\citep{strehl2006pac, bellemare2016unifying, schmidhuber1991possibility} and therefore are more difficult to optimise as well as they do not come with a straightforward way to inject structure.
Yet another approach is to try finding policies that look for the stochastic shortest path~\citep{tarbouriech2019no,cohen2020near}. A more global approach is taken by~\cite{jin2020reward-free} which defines reward-free reinforcement learning as a quest for generating a given number of reward-free trajectories.

\GEM can be seen as a noise-contrastive approach~\citep{ContrastiveHyvarinen,ozair2019wasserstein} for MSVE, where both positive and negative distributions are built upon the same underlying distribution~\citep{wu2018unsupervised}. There are key innovations in the design of the \GEM objective through which we generalise the noise-contrastive approach to tackle the MSVE exploration problem at scale. In particular, {\bf (i)} \GEM uses different losses for the positive and negative examples, resulting in an \textit{asymmetric loss} that maximises the Shannon entropy, and {\bf (ii) } \GEM deploys a similarity function between pairs of states through which \GEM can generalize its estimation of state-visitation distributions to continuous spaces.

Finally,  \AR regularisation has been previously used to learn the state representation in the  control literature~\citep{lesort2018state} under the name \emph{slowness feature analysis}~\citep{wiskott2002slow, bohmer2015autonomous}. However in prior work \AR is used as an auxiliary task for RL or imitation learning. Whereas in \GEM  \AR regularisation is an indispensable component of the main algorithm.

\section{Conclusion}

In this paper, we introduced \GEM, \GEMFullName, a new approach for exploration in RL, with two main contributions. First, our key theoretical contribution is casting geometry-aware maximum state-visitation entropy as the tractable problem of optimising a simple, novel noise-contrastive objective. Second, to capture the geometry of the domain, we equip \GEM with a similarity function, learned through \textit{Adjacency Regularisation} (\AR), a simple regularisation term added to the \GEM objective. \AR allows us to learn a similarity function that preserves the temporal structure of the MDP in our embeddings. We have numerically evaluated \GEM on discrete- and continuous- state sparse reward problems where we can analyse and measure exploratory behaviour, and we have shown that our approach is efficient and robust to noise. By maximising the overall entropy of the state-visitation distribution, we allow for both inter and intra-episode diversity of states.

For future work, we hope to scale up \GEM to larger domains, and investigate how to capture more complex structure. Another direction may be to extend \GEM to partially observable domains. Finally, \GEM's exhaustive exploration may be useful for reward-free RL~\citep{jin2020reward-free}.

\subsubsection*{Acknowledgements}

Thanks to Pierre Richemond for discussions on high dimensional embedding learning and projection methods related to Adjacency Regularisation.

\bibliographystyle{plainnat}
\bibliography{main.bib}

\clearpage
\appendix

\onecolumn
\appendix


\section{Proofs and Derivations}
\label{app:proofs}

\subsection{From NCE to the {\GEM} Objective Function}

The main \GEM idea can be traced back to noise-contrastive estimation (NCE) \citep{ContrastiveHyvarinen}. The core idea of NCE is to learn to contrast between two separate distributions, with positive examples drawn from $ p^+$ and negative examples drawn from $p^-$. For this we often use a classifier $h(z)=\frac {g(z)}{1+g(z)}$, with $g(z)$ is some positive function often paramterised as $g(z)=e^{s(z)}$, where  $s(z)$ are the classifier logits. The objective of NCE is then a binary logistic regression maximisation:
\begin{align}
\label{eq:ce.loss}
    \mathrm{NCE} \triangleq \mathbb{E}_{z\sim p^+} \left[\log \left(h(z)\right)\right] + \mathbb{E}_{z\sim p^-}\left[ \log \left(1-h(z)\right)\right].
\end{align}
Then at the optimum, $g(z)$ will converge to the ratio of the probabilities $\frac{p^+(z)}{p^-(z)}$. In order to extract a single probability rather than a ratio of probabilities, we need to pick specific positive and negative distributions. Similar to \citep{wu2018unsupervised}, we turn the contrastive objective into an auto-contrastive objective where we pick positive examples $z$ to be the fully correlated pair $(x,x)$ of samples, and the negative examples to be the pair of fully independent samples $(x, x')$ of some random variable $x, x' \sim p$. Then intuitively, the probability of positive examples is $p(x)$, the probability of negative examples is $p(x)p(x)$, and thus the ratio is $g(z)=g(x, x) = \frac{p(x)}{p(x)p(x)}=1/p(x)$. Thus the auto-contrastive objective can be used to estimate  probability distributions.

In the context of maximum entropy exploration the cross-entropy objective of \cref{eq:ce.loss} can be used to obtain estimates of state-visitation probabilities in the   auto-contrastive regime. However the auto-contrastive objective based on \cref{eq:ce.loss} does not represent a standard notion of entropy by which we desire to tackle the problem of maximum state-visitation entropy (MSVE) exploration. In order to connect back to Shannon entropy, we need to significantly modify the auto-contrastive objective function and turn to using an asymmetric, objective between positive and negative examples as follows
\begin{align}
\text{\GEM}(h, \pi) &\triangleq  \mathbb{E}_{x\sim p^{\pi}} \left[\ln(h(x,x)) \right] - \mathbb{E}_{x,x' \sim p^{\pi}} \left[ h(x,x') \right]+1.
\end{align}
where $h$ is constrained to be non-negative. This is the origin of \cref{eq:taco.discrete.gen}.
As we have shown in \cref{thm:taco.disc.gen}, maximising this objective gives us the MSVE policy $\pi^\star_E$.

\subsection{Proof of \Cref{thm:taco.disc.gen}}
\tacodisc*
\begin{proof}
The objective is restated here:
\begin{align}
\text{\GEM}(h, \pi) &\triangleq  \mathbb{E}_{x\sim p^{\pi}} \left[\ln(h(x,x)) \right] - \mathbb{E}_{x,x' \sim p^{\pi}} \left[ h(x,x') \right]+1.
\end{align}
To find the maximiser of the objective, we break down the maximization to first maximise for $h$, and then maximise for $\pi$:
\begin{align}
\max_{h, \pi} \text{\GEM}(h, \pi) &= \max_{\pi} \max_{h} \text{\GEM}(h, \pi)
\end{align}
So first we try to maximise $h$. Since $\X$ is discrete, we can consider $h(x, x')$ a matrix of variables, where we have a variable for every $x, x' \in \X$. Then we write out the expectations in terms of sums:
\begin{align}
\text{\GEM}(h, \pi) =  \sum_x p^{\pi}(x) \left[\ln h(x,x) \right] -  \sum_x p^{\pi}(x) \sum_{x'} p^{\pi}(x') h(x,x') + 1.
\end{align}

Note that $h(x, x')$ for $x \neq x'$ only appears as a negative term as the second term. Thus, the maximum of zero can be attained by setting $h(x, x') = 0$. Thus we can simplify the objective by only considering the variables when $x = x'$, and the letting $g_x = h(x, x)$:
\begin{align}
\sum_x p^{\pi}(x) \left[ \ln g_x \right] -  \sum_x p^{\pi}(x)p^{\pi}(x) g_x + 1.
\end{align}
Since $g_x$ is a separate variable for every $x$, we can maximise this sum by maximising every term of the sum w.r.t. $x$, reducing this to the following single variable problem:
\begin{align}
\max_{g_x} \; p^{\pi}(x) \left[ \ln g_x \right] -  p^{\pi}(x)p^{\pi}(x) g_x.
\end{align}

First, note that when $p^{\pi}(x) = 0$, the objective becomes a constant, and therefore does not matter when trying to maximise $g(x)$; $g(x)$ can take on any value. Thus from now on we only consider $x$ such that $p^{\pi}(x) > 0$.

Then, we can find the critical points by setting the derivative to zero:
\begin{align}
0 &= \frac{d}{d g_x} \left[ p^{\pi}(x) \left[ \ln g_x \right] -  p^{\pi}(x)p^{\pi}(x) g_x \right] \\
\implies 0 &= \frac{1}{g_x} p^{\pi}(x) -  p^{\pi}(x)p^{\pi}(x)  \\
\implies g_x &=  \frac{1}{p^{\pi}(x)}
\end{align}
To see what kind of critical point this is, we compute the second derivative:
\begin{align}
& \frac{d^2}{d g_x^2} \left[ p^{\pi}(x) \left[ \ln g_x \right] -  p^{\pi}(x)p^{\pi}(x) g_x \right] \\
&= -\frac{1}{g_x^2} p^{\pi}(x)
\end{align}
Since we are only considering $p^{\pi}(x)$ > 0, then this second derivative $-\frac{1}{g_x^2} p^{\pi}(x) < 0$. This means that the critical point is a local maximum. Furthermore, since there is only one critical point, and the second derivative is always negative, this local maximum is the global maximum. Thus $g^\star_x = \frac{1}{p^{\pi}(x)}$. This also means that
\begin{align}
    h^\star(x, x') &= \1(x = x')g^\star_x \\
    &= \frac{\1(x = x')}{p^{\pi}(x)}
\end{align}
Plugging this back into \GEM:
\begin{align}
\text{\GEM}(h^\star, \pi) &= \sum_x p^{\pi}(x) \left[\ln \frac{1}{p^{\pi}(x)} \right] -  \sum_x p^{\pi}(x) \sum_{x'} p^{\pi}(x') \frac{\1(x = x')}{p^{\pi}(x)} + 1 \\
&= \sum_x p^{\pi}(x) \left[\ln \frac{1}{p^{\pi}(x)} \right] -  \sum_x p^{\pi}(x) p^{\pi}(x) \frac{1}{p^{\pi}(x)} + 1 \\
&= \sum_x p^{\pi}(x) \left[- \ln p^{\pi}(x) \right] \\
&= H(p^{\pi})
\end{align}
We get that maximising w.r.t. $h$ results in simply the Shannon entropy of $p^{\pi}$. Finally, we put it together to get:
\begin{align}
    \max_{h, \pi} \text{\GEM}(h, \pi) &= \max_{\pi} \max_{h} \text{\GEM}(h, \pi) \\
    &= \max_{\pi} H(p^{\pi})
\end{align}
Thus the maximiser for $\pi$ is the optimal Shannon MSVE policy $\pi^\star_E = \underset{\pi}{\mathrm{argmax}} \, H(p^{\pi})$.
\end{proof}

\subsection{Proof of \Cref{thm:taco.disc}}
\tacodisccor*
\begin{proof}
This corollary follows immediately from the proof of \Cref{thm:taco.disc.gen}.
\end{proof}

\subsection{Proof of \Cref{thm:taco.kernel}}
\tacokernel*
\begin{proof}
We restate the full, geometry-aware \GEM objective function here:
\begin{align}
\text{\GEM}_{k}(g, \pi) &\triangleq  \E_{x\sim p^{\pi}} \left[ \ln(g(x)) \right] -\E_{x,x' \sim p^{\pi}} \left[k(x, x') g(x) \right] + 1.
\end{align}
Recall that the similarity profile is $p_k^{\pi}(x) = \E_{x' \sim p^\pi} [k(x, x')] $. Then we can slightly rewrite the \GEM as follows:
\begin{align}
\text{\GEM}_{k}(g, \pi) &\triangleq  \E_{x\sim p^{\pi}} \left[ \ln(g(x)) \right] -\E_{x \sim p^{\pi}} \left[ g(x) p_k^{\pi}(x) \right] + 1.
\end{align}
We then follow the same proof structure as for the proof for \Cref{thm:taco.disc.gen}, to first prove the result for the finite case. We first decompose the maximisation into first maximising for $g$, and then maximising for $\pi$:
\begin{align}
    \max_{g, \pi} \text{\GEM}_k(g, \pi) &= \max_{\pi} \max_{g} \text{\GEM}_k(g, \pi)
\end{align}
Then we reduce the problem of maximising $g$ to pointwise maximisation of each separate variable $g_x = g(x)$ of the following:
\begin{align}
\max_{g_x} \; p^{\pi}(x) \ln (g_x) - p^{\pi}(x) g_x p_k^{\pi}(x) + 1.
\end{align}
Solving for the critical point we get, similar to before, that:
\begin{align}
    g^\star_x &= \frac{1}{p_k^{\pi}(x)}
\end{align}
where instead of $p^{\pi}(x)$ we now get the similarity profile $p_k^{\pi}(x)$. By the same argument as in the proof for \Cref{thm:taco.disc.gen}, the second derivative is negative, and thus this is the global maximum. Plugging this back in, we get that
\begin{align}
    \max_{g, \pi} \text{\GEM}_k(g, \pi) &= \max_{\pi} \max_{g} \text{\GEM}_k(g, \pi) \\
    &= \max_{\pi} \E_{x\sim p^{\pi}} \left[ - \ln p_k^{\pi}(x) \right] \\
    &= \max_{\pi} H_k(p^{\pi})
\end{align}
Where the maximiser for $\pi$ is $\pi^\star_E = \underset{\pi}{\mathrm{argmax}} \, H_k(p^{\pi})$, the geometry-aware Shannon MSVE policy.

\textbf{Extension to the continuous case.} In this case we start by similarly decomposing the maximisation:
\begin{align}
    \max_{g, \pi} \text{\GEM}_k(g, \pi) &= \max_{\pi} \max_{g} \text{\GEM}_k(g, \pi) \\
    &= \max_{\pi} \max_{g} \E_{x\sim p^{\pi}} \left[ \ln(g(x)) \right] -\E_{x \sim p^{\pi}} \left[ g(x) p_k^{\pi}(x) \right] + 1 \\
    &= \max_{\pi} \max_{g} \E_{x\sim p^{\pi}} \left[ \ln(g(x)) - g(x) p_k^{\pi}(x) \right] + 1.
\end{align}
From the proof of the discrete case, we know that $g^\star_x = \frac{1}{p_k^{\pi}(x)}$ is the pointwise maximiser for $g(x)$ for every $x$ inside the expectation. Thus, this is also the maximiser for the entire expectation, as maximising pointwise is the best that we can do. Therefore the result also holds for continuous distributions (the similarity profile $p_k^{\pi}(x)$ is well-defined when the similarity $k$ is smooth and nicely integrable).
\end{proof}

\subsection{Proof of \Cref{prop:aco.grad}}
\gradprop*
\begin{proof}
We notice that the  gradient term in \cref{eq:grad.si}  is an unbiased empirical estimate of  the gradient of \GEM w.r.t. the parameters of $g$. In the case of $\nabla_{\theta}\GEM_k(g,\pi)$ given the fact that $k$ is a symmetric function of $x$ and $x'$ we have the following from the product rule
\begin{align*}
\nabla_{\theta}\GEM_k(g,\pi)=   \int \nabla_{\theta}P^\pi(x)\left[\ln(g(x))-\E_{x'\sim p^{\pi}}(k(x,x')(g(x)+g(x')))\right]dx.
\end{align*}

From the policy gradient theorem \citep{sutton1999policy} we deduce

\begin{align*}
\nabla_{\theta}\GEM_k(g,\pi)=  \E\left[\sum_{t=1}^{T-1} \nabla_{\theta} \ln(\pi(a_t,x_t))\sum_{\tau=t+1}^T(\ln(g(x_{t+\tau}))-\E_{x'\sim p^{\pi}}\left( k(x_{t+\tau},x')\left(g(x_{t+\tau})+g(x'))\right)\right) \right],
\end{align*}
where the outer expectation is with respect to the stochastic process induced by the policy $\pi$. The result then follows by replacing the expectations with their empirical estimates along the trajectory $(x_1,x_2,\dots,x_T)$.

\end{proof}

\subsection{Proof of \Cref{prop:discrete}}
\propdisc*
\begin{proof}
We restate the full \GEM objective here:
\begin{align}
\text{\GEM}_{k}(g, \pi) &\triangleq  \E_{x\sim p^{\pi}} \left[ \ln(g(x)) \right] -\E_{x,x' \sim p^{\pi}} \left[k(x, x') g(x) \right] + 1.
\end{align}
We know from \Cref{thm:taco.kernel} that
\begin{align}
    \max_{g, \pi} \text{\GEM}_{k}(g, \pi) &= \max_{\pi} H_k(p^{\pi}) \\
    &= \max_{\pi} \E_{x\sim p^{\pi}} \left[ - \ln p_k^{\pi}(x) \right]
\end{align}
where $g^\star = \frac{1}{p_k^{\pi}(x)}$. Then let's further maximise w.r.t. $k$:
\begin{align}
    \max_{k, g, \pi} \text{\GEM}_{k}(g, \pi) &= \max_{\pi} \max_{k} \E_{x\sim p^{\pi}} \left[ - \ln p_k^{\pi}(x) \right] \\
    &= \max_{\pi} \max_{k} \E_{x\sim p^{\pi}} \left[ - \ln \E_{x' \sim p^{\pi}} [k(x, x')] \right] \\
    &= \max_{\pi} \max_{k} \E_{x\sim p^{\pi}} \left[ - \ln \left( \sum_{x'} p^{\pi}(x') k(x, x') \right) \right]
\end{align}
In order to maximise this expression, we want to minimise the sum inside the logarithm. However we cannot simply set $k(x, x') = 0$, as we are constrained by requiring that $k(x, x) = 1$. Therefore we let $k(x, x) = 1$, and for all $x \neq x'$, we can set $k(x, x') = 0$. This is equivalent to setting $k(x, x') = \1(x = x')$. Thus
\begin{align}
    k^\star(x, x') = \1(x = x')
\end{align}
and plugging that in we get:
\begin{align}
    \max_{k, g, \pi} \text{\GEM}_{k}(g, \pi) &= \max_{\pi} \max_{k} \E_{x\sim p^{\pi}} \left[ - \ln \left( \sum_{x'} p^{\pi}(x') k(x, x') \right) \right] \\
    &= \max_{\pi} \E_{x\sim p^{\pi}} \left[ - \ln p^{\pi}(x) \right] \\
    &= \max_{\pi} H(p^{\pi}) \\
    &= H(p^{\pi^\star_E})
\end{align}
and we recover the discrete Shannon entropy.
\end{proof}

\subsection{Generalization of {\GEM} to Tsallis Entropy}

\textbf{Tsallis Entropy.} \emph{Tsallis entropy} is defined as \citep{tsallis1988possible}:
\begin{align}
    H_{\alpha}(p) \triangleq \frac{1}{\alpha-1}\left( 1 - \mathbb{E}_{x \sim p} \left[ p(x)^{\alpha-1} \right] \right)
\end{align}
for the real $\alpha$. Taking the limit as $\alpha \rightarrow 1$, this simplifies to Shannon entropy. Tsallis entropy is sometimes called a \textit{pseudo-entropy} since it satisfies all the properties of the standard entropy except additivity~\citep{tsallis1988possible}. Here, we extend this definition to a geometry-aware version of Tsallis entropy (similar to geometry-aware Shannon entropy) by replacing $p$ with its similarity profile $p_k$:
\begin{align}
    H_{\alpha,k}(p) \triangleq \frac{1}{\alpha-1}\left( 1 - \mathbb{E}_{x \sim p} \left[ p_k(x)^{\alpha-1} \right] \right)
\end{align}

We now introduce the generalisation of the \GEM objective function:
\begin{align}
\text{\GEM}_{\alpha, k}(g, \pi) &\triangleq  \frac{1}{\alpha - 1} + \left(1 - \frac{1}{\alpha - 1}\right) \E_{x\sim p^{\pi}} \left[ g(x)^{1 - \alpha} \right] -\E_{x,x' \sim p^{\pi}} \left[k(x, x') g(x)^{2 - \alpha} \right]
\end{align}
where $\alpha < 2$. For $\alpha = 1$, we take the limit of $\left( \alpha - 1 \right) \rightarrow 0$ and use the identity $\lim_{\delta \rightarrow 0} \frac{x^{\delta} - 1}{\delta} = \ln x$ to recover the Shannon case:
\begin{align}
\text{\GEM}_{1, k}(g, \pi) &\triangleq \E_{x\sim p^{\pi}} \left[ \ln(g(x)) \right] -\E_{x,x' \sim p^{\pi}} \left[k(x, x') g(x) \right] + 1.
\end{align}
Next, for $\alpha \neq 1$, we follow the structure of the proof for the Shannon case. We first decompose the maximisation:
\begin{align}
    \max_{g, \pi} \text{\GEM}_{\alpha, k}(g, \pi) &= \max_{\pi} \max_{g} \text{\GEM}_{\alpha, k}(g, \pi)
\end{align}
Then we proceed to do a pointwise maximisation for $g_x = g(x)$:
\begin{align}
    \max_{g_x} \; \left(1 - \frac{1}{\alpha - 1}\right) p^{\pi}(x)  g_x^{1 - \alpha} - p^{\pi}(x) g_x^{2 - \alpha}\E_{x' \sim p^{\pi}} \left[k(x, x')\right]
\end{align}
Note again that for $p^{\pi}(x) = 0$, this becomes a constant and $g_x$ can take on any value. Thus we restrict ourselves to the case where $p^{\pi}(x) > 0$. We find the critical point by solving for the zeros of the derivative:
\begin{align}
0 &= \frac{d}{d g_x}  \left(1 - \frac{1}{\alpha - 1}\right) p^{\pi}(x)  g_x^{1 - \alpha} - p^{\pi}(x) g_x^{2 - \alpha}\E_{x' \sim p^{\pi}} \left[k(x, x')\right] \\
\implies 0 &=  (1 - \alpha)\left(1 - \frac{1}{\alpha - 1}\right) p^{\pi}(x)  g_x^{- \alpha} - (2 - \alpha)p^{\pi}(x) g_x^{1 - \alpha}\E_{x' \sim p^{\pi}} \left[k(x, x')\right] \\
\implies 0 &= g_x^{- \alpha} - g_x^{1 - \alpha}\E_{x' \sim p^{\pi}} \left[k(x, x')\right] \\
\implies g_x &=  \frac{1}{\E_{x' \sim p^{\pi}} \left[k(x, x')\right]} = \frac{1}{p_k^{\pi}(x)}
\end{align}
Thus the critical point is identical to the Shannon case. Next we examine the second derivative:
\begin{align}
& \frac{d^2}{d g_x^2}  \left(1 - \frac{1}{\alpha - 1}\right) p^{\pi}(x)  g_x^{1 - \alpha} - p^{\pi}(x) g_x^{2 - \alpha}\E_{x' \sim p^{\pi}} \left[k(x, x')\right] \\
&=  - \alpha(1 - \alpha)\left(1 - \frac{1}{\alpha - 1}\right) p^{\pi}(x)  g_x^{- \alpha-1} - (1 - \alpha)(2 - \alpha)p^{\pi}(x) g_x^{- \alpha}\E_{x' \sim p^{\pi}} \left[k(x, x')\right] \\
&=  (2 - \alpha) p^{\pi}(x) g_x^{\alpha + 1} \left(- \alpha - (1 - \alpha) g_x \E_{x' \sim p^{\pi}} \left[k(x, x')\right] \right)
\end{align}
We plug in the critical point and this second derivative simplifies to:
\begin{align}
    & (2 - \alpha) p^{\pi}(x) \left( \frac{1}{\E_{x' \sim p^{\pi}} \left[k(x, x')\right]} \right)^{\alpha + 1} \left(- \alpha - (1 - \alpha) \left( \frac{1}{\E_{x' \sim p^{\pi}} \left[k(x, x')\right]}\right) \E_{x' \sim p^{\pi}} \left[k(x, x')\right] \right) \\
    &= -(2 - \alpha) p^{\pi}(x) \left( \frac{1}{\E_{x' \sim p^{\pi}} \left[k(x, x')\right]} \right)^{\alpha + 1} \\
    &< 0
\end{align}
So this critical point is a maximum as long as $\alpha < 2$. Thus the maximiser $g^\star$ is exactly the same as for the Shannon case:
\begin{align}
    g^\star(x) &= \frac{1}{p_k^{\pi}(x)}
\end{align}
Plugging this back into the full objective we get:
\begin{align}
    \max_{g, \pi} \text{\GEM}_{\alpha, k}(g, \pi) &= \max_{\pi} \max_{g} \text{\GEM}_{\alpha, k}(g, \pi) \\
    &= \max_{\pi}  \frac{1}{\alpha - 1} + \left(1 - \frac{1}{\alpha - 1}\right) \E_{x\sim p^{\pi}} \left[ \left( \frac{1}{p_k^{\pi}(x)} \right)^{1 - \alpha} \right] -\E_{x,x' \sim p^{\pi}} \left[k(x, x') \left( \frac{1}{p_k^{\pi}(x)} \right)^{2 - \alpha} \right] \\
    &= \max_{\pi}  \frac{1}{\alpha - 1} + \left(- \frac{1}{\alpha - 1}\right) \E_{x\sim p^{\pi}} \left[ \left( \frac{1}{p_k^{\pi}(x)} \right)^{1 - \alpha} \right] \\
    &= \max_{\pi}  \frac{1}{\alpha - 1} \left( 1 - \E_{x\sim p^{\pi}} \left[ p_k^{\pi}(x)^{\alpha - 1} \right] \right) \\
    &= \max_{\pi} H_{\alpha,k}(p^{\pi})
\end{align}
Thus we are able to find the geometry-aware Tsallis MSVE policy with this generalised \GEM.

\clearpage

\section{Experiment Details}
\label{app:expdetails}

\subsection{1D bimodal Gaussian Learning for \Cref{sec:similarity}}

\begin{algorithm}
 \KwInputs{$x$, $n_{\mathrm{bucket}}$, $m_{\min}$, $m_{\max}$}
 Let $y \triangleq \frac{x - m_{\min}}{m_{\max} - m_{\min}}$ \;
 Let $b \triangleq \{ 0.5, 1.5, 2.5, \dots, n_{\mathrm{bucket}} - 0.5 \}$ \;
 Return $\{ e^{-n_{\mathrm{bucket}}\vert b_i - y \vert} \mid b_i \in b \}$ \;
 \caption{soft1hot($x$)}
 \label{alg:soft1h}
\end{algorithm}

\begin{algorithm}
 \KwInputs{Minibatches $B_1=\{x_i\}$, $B_2=\{x_j\}$, neural networks $g$, $f$, scalars $c$, $n_{\text{neg}}$}
 Let $g_x \triangleq \text{softplus}(g(x)) + 10^{-8}$ \;
 Let $k(x, x') \triangleq e^{-c\|f(x) - f(x')\|_2}$ \;
 Let $\text{REG}(x) \triangleq \|f(x)\|_2^2 $ \;
$R_{\GEM}(x_i) \triangleq 1 + \log g_{x_i} -  \left( g_{x_i} + g_{x'} \right) \frac{1}{n_{\text{neg}}}\sum_{m=1}^{n_{\text{neg}}} k(x_i,x_m')$ \;
\qquad where $x_m' \sim \text{RandomSample}(B_2)$ \;
Then $\text{LOSS} \triangleq \frac{1}{|B_1|}\sum_{i=1}^{|B_1|} \left( 1 + \log g_{x_i} -  g_{x_i} \frac{1}{n_{\text{neg}}}\sum_{m=1}^{n_{\text{neg}}} k(x_i,x_m') + w_{\text{reg}} \text{REG}(x_i) \right)$ \;
\qquad where $x_m' \sim \text{RandomSample}(B_2)$ \;
 Return ($R_{\GEM}(x_i)$, $\text{LOSS}$) \;
 \caption{GEMLoss($B_1,B_2$)}
 \label{alg:taco}
\end{algorithm}

In this part, we provide experimental details for \GEM in \cref{fig:tacogaussian} given in \cref{sec:similarity}. The loss is computed in \cref{alg:taco} and optimised with the \ADAM optimiser. Note that we use the softplus activation on top of $g$ to make sure that it always returns something positive. We also add a small regularisation to the norm of the output of $f$ (REG) to prevent divergence to infinity.

Next, we outline the data generating distribution. Let $N, N' \sim \mathrm{TruncNorm}(0, 1)$ be a random variable from the standard truncated normal normal distribution with mean $0$, variance $1$, truncated between $\{-2, 2\}$. Let $U \sim \mathrm{Uniform}(0, 1)$ be a random variable from the uniform distribution in $[0, 1]$. Then the bimodal distribution used for generating data is
\begin{align*}
    X &\triangleq \frac{30}{8} \left( \mathbf{I}(U < 0.3)N_1 + \mathbf{I}(U \geq 0.3)N_2 + 4 \right) \\
    N_{1} &= N - 2 \\
    N_{2} &= N' + 2
\end{align*}
$X$ has been scaled and shifted so that its support is $[0, 30]$. The discretised version has $30$ points equally spaced in $[0, 30]$, where the probability is computed by normalizing the density values at those $30$ points. In order to take full advantage of neural networks, we use a soft-one-hot encoding (\cref{alg:soft1h}) of the scalar input before passing it to the inner layers of the neural networks of $f$ and $g$. Experiment hyperparameters are outlined in \cref{tab:gaussianhyper}.

\begin{table}
    \centering
    \begin{tabular}{c||c|c}
         & Fixed $k$ & Trained $k$ \\
         \hline \hline
         $f$ & identity & MLP[soft1hot, Linear(128), relu, Linear(128), relu, Linear(64)] \\
         \hline
         $c$ & $2$ & $1$ \\
         \hline
         Batch size & \multicolumn{2}{c}{$256$} \\
         \hline
         $\beta$ & \multicolumn{2}{c}{$0$} \\
         \hline
         $n_{\mathrm{neg}}$ & \multicolumn{2}{c}{$8$} \\
         \hline
         $w_{\mathrm{reg}}$ & \multicolumn{2}{c}{$10^{-6}$} \\
         \hline
         $n_{\mathrm{bucket}}$ & \multicolumn{2}{c}{$30$} \\
         \hline
         $m_{\min}$ & \multicolumn{2}{c}{$0$} \\
         \hline
         $m_{\max}$ & \multicolumn{2}{c}{$30$} \\
         \hline
         $g$ & \multicolumn{2}{c}{MLP[soft1hot, Linear(128), relu, Linear(128), relu, Linear(1)]} \\
         \hline
         Training Steps & \multicolumn{2}{c}{$1000$} \\
         \hline
         Optimiser & \multicolumn{2}{c}{Adam(learning rate = $10^{-3}$, $\beta_1 = 0$, $\beta_2 = 0.95$)}
    \end{tabular}
    \vspace{0.5em}
    \caption{Hyperparameters for 1D Bimodal Gaussian}
    \label{tab:gaussianhyper}
\end{table}

\subsection{{\GEM} Experiment Details}

The plots in the experiments (\cref{sec:experiments}) are computed by splitting the x-axis into 20 buckets of equal size, and computing the average of each bucket. This is done to smooth the data for each individual run. Then 5 runs are averaged together, and the standard error with 95\% confidence interval is computed and shown as the shaded region.

The dynamics of the gridworlds in \cref{fig:envs} have 5 standard actions of [no-op, up, down, left, right]. The initial state of the agent is picked uniformly at random from any of the blue square locations. The blue locations not chosen are normal free blocks. In each episode, a random green square location is chosen and a reward block (with reward $1.0$) is placed there. The green locations not chosen are normal free blocks. The episode ends after the agent moves over the reward block and sees the reward.

The episode length for \texttt{2-Rooms} is $30$, and \texttt{16-Leaves} is $18$. The state representation of these environments are pixel image arrays, i.e., 3D arrays of shape $(\mathrm{width}, \mathrm{height}, 3)$, and are shown in \cref{fig:envobs}; they consists of an upper world map, which shows a map of the rooms as well as which room the agent is in and which room the reward is in, and a lower room map, which shows the map of the current room. In the noisy version of \texttt{2-Rooms}, there is a square with random red and green components ($256^2$ different colours) in-between the upper and lower maps. Each square is $8 \times 8$ pixels.

The visitation entropy (\cref{fig:sweepdist}) is tracked by keeping an exponential moving average with a decay of $0.99$ of the visitation count for each open square, and then computing the empirical Shannon entropy over these counts. The heatmaps (\cref{fig:sweepdist}) are generated using the same counts, but are normalised by dividing all counts by the largest count so that the most visited square has value 1, corresponding to pure black.

The environments Cartpole Swingup and Mountain Car are standard, continuous-state, sparse reward domains with episode length 1000. We use the implementations of \citep{osband2020bsuite}.

A training step is outlined in \cref{alg:moretacomax}. The \AR loss is computed from \cref{alg:arloss}. We compute the \GEM intrinsic reward from \cref{alg:taco}, normalise it, and mix it with the extrinsic reward to use for policy gradient (\cref{alg:policygradient}) to train $\pi$ and $V$, which is a standard actor-critic algorithm. $g$ and $f$ are optimised directly with the computed loss.

For all experiments, we use a multi-process agent, where we have 64 processes running the agent policy and gathering data from the environment in parallel. The data is gathered and sent to another processes that computes the gradients and updates the parameters.

The basic network architecture used is shown in \cref{tab:commonhyper}, along with training hyperparameters. For the continuous domains of Cartpole Swingup and Mountain Car, we use a slightly different architecture since their state is not an image (\cref{tab:continuoustorso}). The RNN cores for the value $V$ and policy $\pi$ networks take as input the concatenation of the torso output (torso is applied to the state), a 1-hot representation of the previous action, the previous extrinsic reward, and a soft 1-hot representation of the current timestep index. The embedding function $f$ and the $g$ function only take state as input, without concatenating other quantities. The batch size in \cref{tab:commonhyper} denotes the total batch size, i.e. it is the combination of $B_1$ and $B_2$ in \cref{alg:moretacomax}, each of which is half of the size. Each minibatch contains batch size number of traces, which are short segments of episodes. The trace period denotes which indices the traces start. For example, a trace length of $20$ with a trace period of $10$ means that the traces are timesteps $(1,2, \dots, 20)$, $(10,11,\dots,30)$, $(20, 21, \dots, 40)$, etc.. The traces are randomly chosen from within the episode so that the traces are not all in lockstep within a minibatch. The trace lengths and trace periods are shown in \cref{tab:envhyper}.

The empirical oracle MSVE baseline (\cref{sec:sweep2timescale}) shares the same policy gradient algorithm as \GEM, except that the intrinsic reward is computed as $(-\ln \mathrm{count}(x))$, where $\mathrm{count}(x)$ is an exponential moving average of the true count of state $x$, with decay $0.99$. The intrinsic reward is them normalised in the same way and summed with the extrinsic reward.

The baseline RND shares the same policy gradient algorithm as \GEM, except that the intrinsic reward is computed through the RND method. The network used for the fixed random target and the predictor is the same as torso in \cref{tab:commonhyper}. Reward and observation normalisation are applied as per RND, with the exponential decay being $0.95$. The bonus is then scaled by $0.02$ before combined with the extrinsic reward.

The baseline NGU also shares the same policy gradient algorithm as \GEM, except that the intrinsic reward is computed through the NGU method. The only modification we apply is scaling the intrinsic reward before combining it with the extrinsic reward. The state embedding is trained using action prediction as in \citep{Badia2020Never}, with an action predictor network that is an MLP with a relu hidden layer of size 256 with a linear output of size the same as the number of actions.

The environment specific parameters and scaling is show in \cref{tab:envhyper}.

\begin{algorithm}
 \KwInputs{Minibatch $B = \{ (x_{t}, a_t, x_{t+1}) \}$, neural network $f$, Huber exponent $q$, Huber offset $\delta$}
 Let $\mathrm{LOSS}_t \triangleq \left(\delta^q + \|f(x_t) - f(x_{t+1})\|_2^q \right)^{1/q}$ \;
 Then $\mathrm{LOSS} \triangleq \frac{1}{|B|}\sum_{t} \mathrm{LOSS}_t $ \;
 Return $\mathrm{LOSS}$
 \caption{ARLoss($B$)}
 \label{alg:arloss}
\end{algorithm}

\begin{algorithm}
 \KwInputs{Minibatch of trajectories $\{(x_t^{(i)}, a_t^{(i)}, r_t^{(i)})\}$, episode length $T$, policy logits $\pi$, value function $V$, action entropy cost $w_{\mathrm{ent}}$}
 Let $\mathrm{TRACE}^{(i)}(t, m) \triangleq \left( \sum_{j=t}^{t + m} r_j^{(i)} \right) + V(x_{t + m + 1}^{(i)})$ \;
 Let $\mathrm{RET}^{(i)}(t) \triangleq \frac{1}{T - t} \sum_{m=0}^{T - t - 1} \mathrm{TRACE}^{(i)}(t, m)$ \;
 Let $\mathrm{VLOSS}^{(i)}(t) \triangleq \|V(x_t^{(i)}) - \mathrm{StopGradient}(\mathrm{RET}^{(i)}(t))\|_2^2$ \;
 Then $\mathrm{VLOSS} \triangleq \frac{1}{|B|}\sum_{i, t} \mathrm{VLOSS}^{(i)}(t) $ \;
 Let $\mathrm{ENT}^{(i)}_t \triangleq \mathrm{ShannonEntropy}(\pi(x_t^{(i)}, \cdot)) $\;
 Then $\mathrm{ENT} \triangleq \frac{1}{|B|}\sum_{i, t} \mathrm{ENT}^{(i)}_t$ \;
 Let $ \mathrm{PLOSS}^{(i)}(t) \triangleq -\pi(x_t^{(i)}, a_t^{(i)}) \cdot \mathrm{StopGradient}(r_t^{(i)} + V(x_{t+1}^{(i)}) - V(x_t^{(i)})) $ \;
 Then $\mathrm{PLOSS} \triangleq \frac{1}{|B|}\sum_{i, t} \mathrm{PLOSS}^{(i)}_t$ \;
 Return $\left( \mathrm{PLOSS} + \mathrm{VLOSS} - w_{\mathrm{ent}} \mathrm{ENT} \right)$
 \caption{PolicyGradient($B$)}
 \label{alg:policygradient}
\end{algorithm}

\begin{algorithm}
 \KwInputs{Minibatches of trajectories $B_1 = \{(x_i, a_i, r_i)\}$ and $B_2$, intrinsic reward mean $m_{\GEM}$, intrinsic reward scale $s_{\GEM}$, adjacency regularisation scale $C$}
 Let ($R_{\GEM, 1}(x_i)$, $\text{LOSS}_1$) = $\mathrm{GEMLoss}(B_1, B_2)$ \;
 Let ($R_{\GEM, 2}(x_i)$, $\text{LOSS}_2$) = $\mathrm{GEMLoss}(B_2, B_1)$ \;
 Let $R_{\GEM}(x_i) = \frac{1}{2}(R_{\GEM, 1}(x_i) + R_{\GEM, 2}(x_i))$ \;
 Let $R_{\GEM}^{\mathrm{normalised}}(x_i) = \frac{R_{\GEM}(x_i) - \mu}{\sigma} \cdot s_{\GEM} + m_{\GEM}$ \;
 \qquad where $\mu$ and $\sigma$ are exponential running average of the mean and std with decay $0.99$ \;
 Let $\text{LOSS} = \frac{1}{2}(\text{LOSS}_1 + \text{LOSS}_2 + C \cdot \text{ARLoss}(B_1) + C \cdot \text{ARLoss}(B_2))$ \;
 Take gradient step of $\min_f \text{LOSS}$ \;
 Take gradient step of $\min_g \text{LOSS} $ \;
 Let $R_{\mathrm{total}}(x_i) \triangleq r_i + R_{\GEM}^{\mathrm{normalised}}(x_i)$ for all $x_i \in B_1 \cup B_2$ \;
 Let $B_{\mathrm{total}} \triangleq \{(x_i, a_i, R_{\mathrm{total}}(x_i))\}$ for all $x_i \in B_1 \cup B_2$ \;
 Take gradient step of $\min_{\pi, V} \mathrm{PolicyGradient}(B_{\mathrm{total}})$ \;
 \caption{Detailed \GEM Training Step}
 \label{alg:moretacomax}
\end{algorithm}

\begin{figure}
    \centering
    \includegraphics[width=1.0\textwidth]{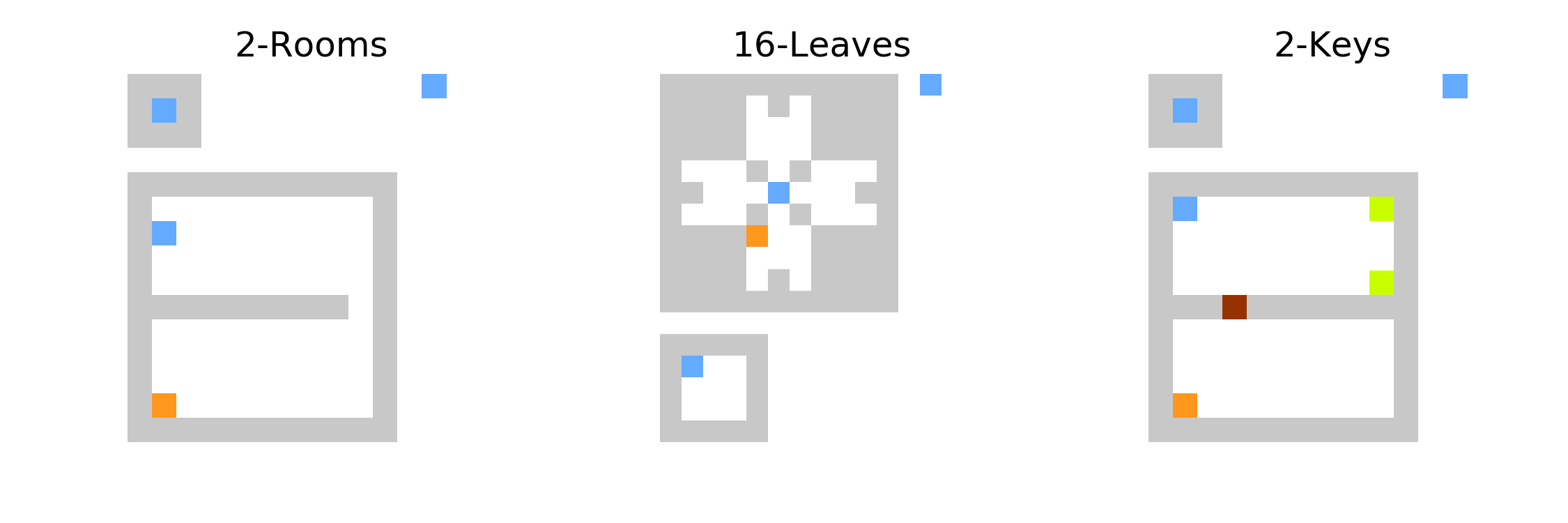}
    \caption{State for \texttt{2-Rooms}, \texttt{16-Leaves} and \texttt{2-Keys}, showing noise block in-between world and room maps. The agent is in blue. The reward square location is orange. Keys are yellow-green. The door is brown. The blue square at the top-right corner of the image indicates that the agent only has one life (unused in our experiments).}
    \label{fig:envobs}
\end{figure}

\begin{table}
    \centering
    \begin{tabular}{c||c}
        \multirow{5}{*}{torso} & Conv2D($32$ channels, kernel size $8 \times 8$, stride $4 \times 4$) \\
        & relu \\
        & Conv2D($32$ channels, kernel size $4 \times 4$, stride $2 \times 2$) \\
        & relu \\
        & Conv2D($64$ channels, kernel size $3 \times 3$, stride $1 \times 1$) \\
        & relu and flatten \\
        & Linear(256) and relu \\
        \hline
        rnntorso & copy of torso \\
        \hline
        fhead & MLP[Linear(256), relu] \\
        \hline
        ghead & copy of fhead \\
        \hline
        rnncore & LSTM(256) \\
        \hline
        pihead & copy of fhead \\
        \hline
        vhead & copy of fhead \\
        \hline
         $f$ & $\mathrm{torso} \circ \mathrm{fhead} \circ \mathrm{Linear(256)}$ \\
         \hline
         $g$ & $\mathrm{torso} \circ \mathrm{ghead} \circ \mathrm{Linear(1)}$ \\
         \hline
         $\pi$ & $\mathrm{rnntorso} \circ \mathrm{rnncore} \circ \mathrm{pihead} \circ \mathrm{Linear(5)}$ \\
         \hline
         $V$ & $\mathrm{rnntorso} \circ \mathrm{rnncore} \circ \mathrm{vhead} \circ \mathrm{Linear(1)}$ \\
         \hline
         Batch size & $256$ \\
         \hline
         $n_{\mathrm{neg}}$ & $32$ \\
         \hline
         $c$ & $1$ \\
         \hline
         $q$ & $4$ \\
         \hline
         $w_{\mathrm{reg}}$ & $10^{-4}$ \\
         \hline
         Training Steps & $100000$ \\
         \hline
         Optimiser & Adam(learning rate = $10^{-4}$, $\beta_1 = 0$, $\beta_2 = 0.95$)
    \end{tabular}
    \vspace{0.5em}
    \caption{Common Hyperparameters}
    \label{tab:commonhyper}
\end{table}

\begin{table}
    \centering
    \begin{tabular}{c||c}
        \hline
        torso & Linear(256) and relu \\
        \hline
    \end{tabular}
    \vspace{0.5em}
    \caption{Cartpole Swingup and Mountain Car Specific Torso}
    \label{tab:continuoustorso}
\end{table}

\begin{table}
    \centering
    \begin{tabular}{c|c|c|c|c}
        & 2-Rooms & 16-Leaves & Cartpole Swingup & Mountain Car \\
        \hline \hline
        $w_{\mathrm{ent}}$ & $10^{-3}$ & $10^{-3}$ & $10^{-2}$ & $10^{-2}$ \\
         \hline
        $m_{\GEM}$ & $0.005$ & $0.005$ & $0.15$ & $0.7$ \\
        \hline
        $s_{\GEM}$ & $0.005$ & $0.005$ & $0.15$ & $0.25$ \\
        \hline
        RND intrinsic reward scale & $0.001$ & $0.001$ & $0.02$ & $0.1$ \\
        \hline
        NGU w/o RND intrinsic reward scale & $0.0003$ & $0.0003$ & $0.02$ & $0.05$\\
        \hline
        NGU w/ RND intrinsic reward scale & $0.0003$ & $0.0003$ & $0.01$ & $0.07$ \\
        \hline
        Episode length & $30$ & $18$ & $1000$ & $1000$ \\
        \hline
        Trace Length & $20$ & $14$ & $20$ & $20$ \\
        \hline
        Trace Period & $10$ & $7$ & $10$ & $10$ \\
    \end{tabular}
    \vspace{0.5em}
    \caption{Env Specific Hyperparameters}
    \label{tab:envhyper}
\end{table}

\clearpage

\section{Additional Experiments}
\label{app:additional_experiments}

\subsection{Irreversible Transitions}
\label{sec:irreversibletransitions}

Because \AR is symmetric, we take a closer look in this section as to what happens when we have asymmetric dynamics. 
We ran \GEM on the simple gridworld illustrated in \cref{fig:2keysenv} (with image representation shown in \cref{fig:envobs}).
The agent spawns at random in one of the blue cells, and the goal is in the green cell. 
To cross the bottleneck brown cell (the ``door''), the agent must first visit at least one of the yellow cells (the ``keys''). 
Visiting a yellow cell (``collecting the key'') turns it into an ordinary cell, as does visiting the brown door cell after having a key (``opening the door'').
Thus both picking up the key and opening the door are asymmetric, irreversible transitions.

\Cref{fig:2keysemb_legend} shows learned embeddings for different states, projected onto the first two principal components.
The different markers denote different configurations of the state (ignoring the agent position), and the colors correspond to different positions on the grid.
For example, the blue circles correspond to the positions on the top room before having collected any keys.
The purple positions correspond to the second room, and we can see that different configurations are embedded differently.
In particular, the embedding groups the lower-room states into three (purple) groups: 
In the left group (+) the agent has only collected the lower key, 
in the center group ($\square$) the agent has collected both keys,
and in the right group (|) the agent has only collected the upper key.
The terminal state (where the green goal is absent) is also embedded separately for each of the three groups.
The blue regions (lower room) are similarly grouped by the configuration of the keys and the door. \Cref{fig:2keysemb_legend} contains a more detailed description of the different configuration and position representations.

The main difference between the symmetric transitions and the asymmetric transitions is that asymmetric transitions results in embedding points that are slightly farther away from each other. This is because we only optimise \AR for one direction as opposed to both directions with symmetric transitions, i.e., the strength of adjacency regularization is halved. Nevertheless, the embeddings learned is still meaningful and \GEM~is still able to explore and solve this task.

\begin{figure}
    \centering
    \includegraphics[width=0.2\linewidth]{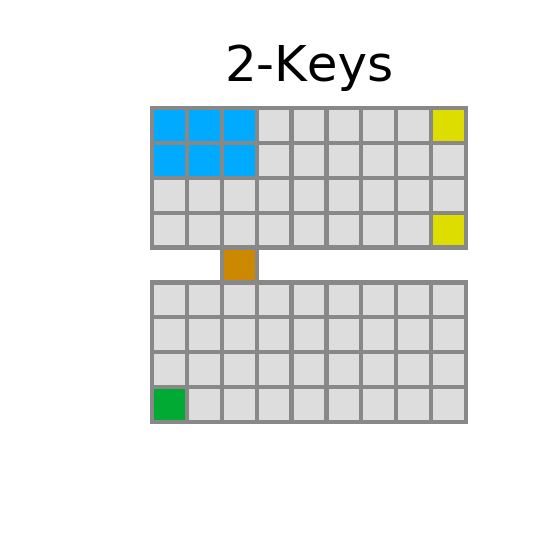}
    \caption{\texttt{2-Keys} Gridworld.}
    \label{fig:2keysenv}
\end{figure}

\begin{figure*}
     \centering
     \includegraphics[width=0.85\linewidth]{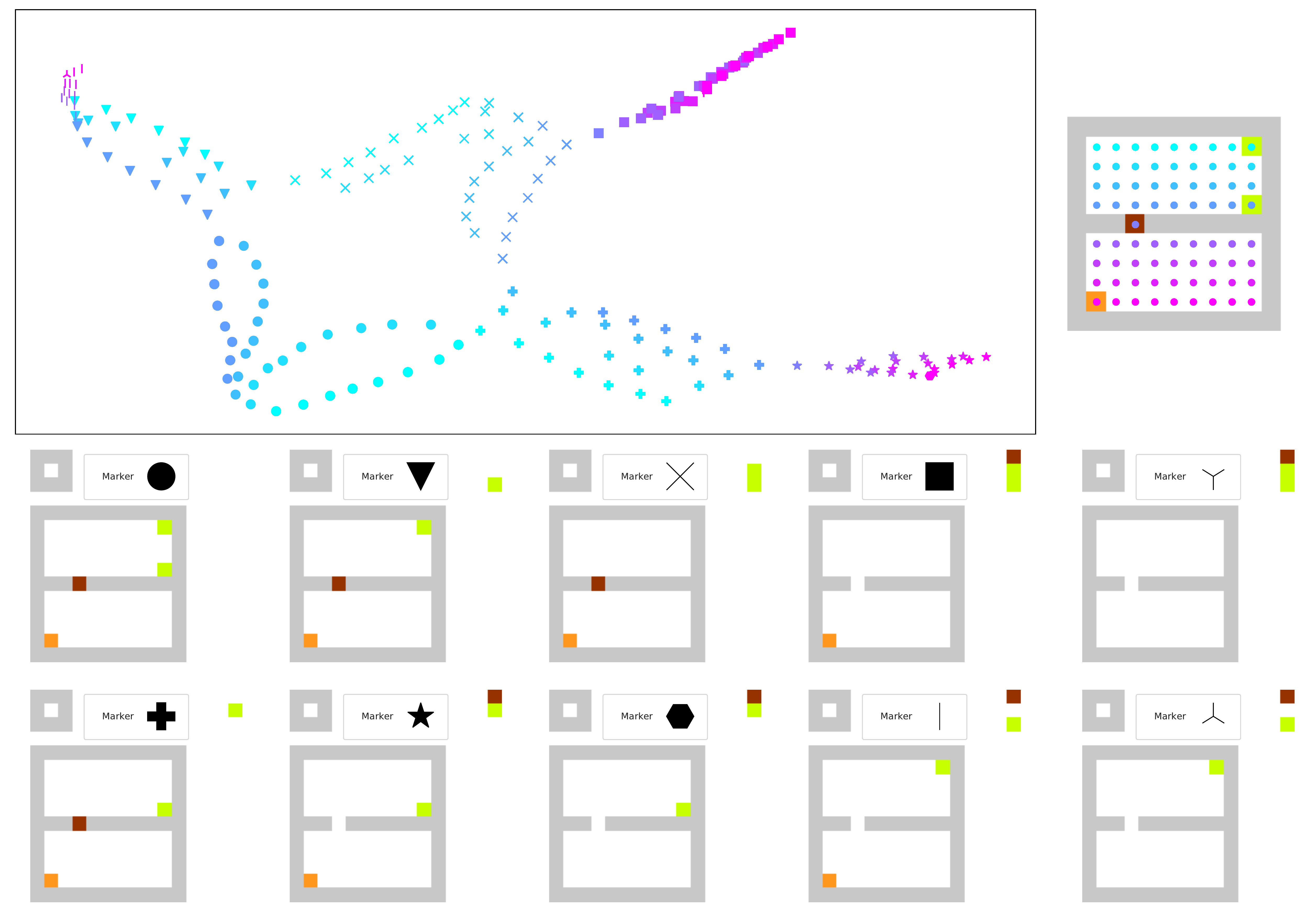}
     \caption{Embeddings learned by \GEM in the \texttt{2-Keys} Gridworld.}
     \label{fig:2keysemb_legend}

\end{figure*}

\end{document}